\documentclass[jornal]{IEEEtran}

\usepackage[utf8]{inputenc} 
\usepackage[T1]{fontenc}    
\usepackage{hyperref}       
\usepackage{url}            
\usepackage{booktabs}       
\usepackage{amsfonts}       
\usepackage{amsmath, bm}
\usepackage{mathrsfs} 
\usepackage{amsthm}
\usepackage{nicefrac}       
\usepackage{microtype}      
\usepackage{xcolor}         

\usepackage{graphicx}
\usepackage{caption}
\usepackage{subcaption}
\usepackage{enumerate}
\usepackage{cite}

\usepackage{tikz}
\usepackage{graphics}

\usepackage[noend]{algorithmic}
\usepackage{algorithm}

\usepackage{amsfonts}
\usepackage{amsmath}
\usepackage{amssymb}
\usepackage{mathtools}
\usepackage{amsthm}
\usepackage{bbm}

\usepackage[capitalize,noabbrev]{cleveref}

\theoremstyle{remark}
\newtheorem{remark}{\bf Remark}[section]

\theoremstyle{plain}
\newtheorem{theorem}{\bf Theorem}[section]

\newtheorem{lemma}{\bf Lemma}[section]

\newtheorem{corollary}{\bf Corollary}[section]

\theoremstyle{definition}

\newtheorem{assumption}{\bf Assumption}[section]

\title{Robust Stochastically-Descending Unrolled Networks}
\author{Samar Hadou, Navid NaderiAlizadeh, and Alejandro Ribeiro 
\thanks{S. Hadou and A. Ribeiro are with the
  Department of Electrical and Systems Engineering,
  University of Pennsylvania. Emails:
  \texttt{\{selaraby, aribeiro\}@seas.upenn.edu}.
  N. NaderiAlizadeh is with the Department of Biostatistics and Bioinformatics, Duke University. Email: \texttt{navid.naderi@duke.edu}}
  }
    \date{}
\begin{document}
\maketitle
\begin{abstract}
Deep unrolling, or unfolding, is an emerging learning-to-optimize method that unrolls a truncated iterative algorithm in the layers of a trainable neural network. However, the convergence guarantees and generalizability of the unrolled networks are still open theoretical problems. To tackle these problems, we provide deep unrolled architectures with a stochastic descent nature by imposing descending constraints during training. The descending constraints are forced layer by layer to ensure that each unrolled layer takes, on average, a descent step toward the optimum during training. We theoretically prove that the sequence constructed by the outputs of the unrolled layers is then guaranteed to converge for in-distribution problems. We then analyze the generalizability to certain out-of-distribution (OOD) shifts in the optimization problems being solved. Our analysis shows that the descending nature imposed by the proposed constraints is transferable under these distribution shifts, subject to a generalization error, thereby providing the unrolled networks with OOD robustness. We numerically assess unrolled architectures trained with the proposed constraints in two different applications, including the sparse coding using learnable iterative shrinkage and thresholding algorithm (LISTA) and image inpainting using proximal generative flow (GLOW-Prox), and demonstrate the performance and robustness advantages of the proposed method.
\end{abstract}


\section{Introduction}

Learning-based approaches have provided unprecedented performance in many applications in the realm of signal processing. Due to the wide availability of data nowadays and recent advancements in hardware infrastructure, these learning-based approaches often outperform their traditional model-based alternatives, where the models are handcrafted based on domain knowledge. This is particularly true because of the large number of parameters used therein that can approximate complicated mappings otherwise difficult to characterize precisely. Furthermore, the number of layers in deep neural networks is usually much smaller than the number of iterations required in a traditional iterative algorithm, and modern computational platforms are optimized to execute the core operations of these networks efficiently. Therefore, learning-based methods provide a computational advantage over their traditional counterparts during inference.

However, deep networks are purely data-driven, and their structures are, therefore, hard to interpret, which is a serious concern in safety-critical applications, such as medical imaging and autonomous driving. In contrast, model-based methods are highly interpretable since they exploit prior domain knowledge to model the underlying physical process. Moreover, the lack of any prior domain knowledge prevents deep networks from generalizing well to unseen problems when abundant high-quality data is not available.

To combine the benefits of the two aforementioned regimes, the seminal work of \cite{gregor_learning_2010} has introduced \emph{algorithm unrolling}, or unfolding. Algorithm unrolling is a \emph{learning-to-optimize} technique that unrolls the iterations of an iterative algorithm through the layers of a customized deep neural network. Under this framework, the unrolled networks can be thought of as an iterative optimizer that is designed based on prior domain knowledge and whose parameters are learnable. Thus these unrolled networks inherit the interpretability of iterative models. Beyond its interpretability, unrolling inversely contributes to the elucidation of deep neural networks by viewing them as the unfolding of an optimization process \cite{yang2022transformers, yu2023white, von2023transformers}. In addition, unrolled networks typically have much fewer parameters compared to traditional deep neural networks, which alleviates the requirement of massive training datasets. An additional advantage of unrolling manifests in accelerated convergence rates, requiring substantially fewer layers (i.e., iterations) while concurrently achieving superior performance compared to standard iterative models \cite{monga_algorithm_2021}. The efficacy of unrolling is evident through its notable success in many applications such as computer vision \cite{zhang2020deep, Wei22, mou2022deep, Li20, qiao2023towards}, wireless networks \cite{hu2020iterative, chowdhury2021unfolding, liu2021deep, Schynol23, huang2023regularization, yang2023knowledge}, medical imaging \cite{li2021deep, nakarmi2020multi, chennakeshava2022deep, wang2023indudonet+}, and even training neural networks \cite{hadou2023stochastic, Ravi2016OptimizationAA}, among many others \cite{hershey2014deep, nasser2022deep, noah2023limited, Liu23}.

Deriving convergence guarantees is an essential theoretical question that arises when developing an iterative algorithm.  
Since unrolled optimizers resemble algorithms that are widely known to be convergent, this might imply convergence guarantees of the unrolled algorithms by default. However, this is not the case. 
The more parameters of the iterative algorithm we set free to learn, the wider the subspace of unrolled models we search and the harder it is to find a convergent model. As a toy example, 
Figure \ref{fig:comp} shows trajectories to a stationary point of a convex function made by both gradient descent (GD) and an unrolled version of it, where we make the descent direction learnable. Unlike gradient descent (left), the unrolled network (middle) moves randomly and in opposite directions, before eventually jumping to a near-optimal point in the last one or two layers. This is not surprising since the training procedure of the unrolled network does not regulate the trajectories across the intermediate unrolled layers. 
\begin{figure*}
    \centering
    \includegraphics[width=0.8\textwidth]{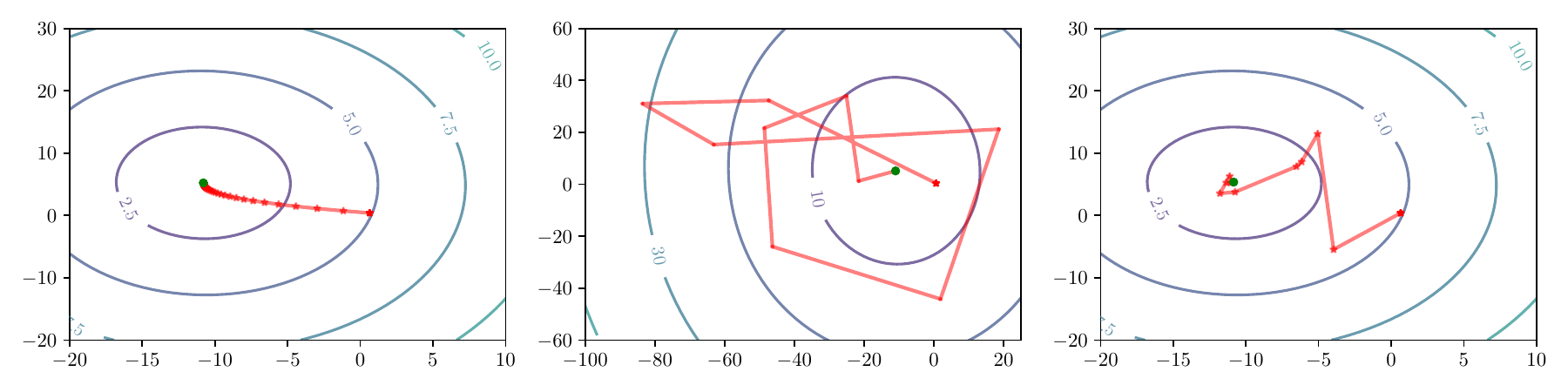}
     \caption{Trajectories of a test example made by (left) gradient descent, (middle) a standard unrolled optimizer, and (right) a constrained unrolled optimizer (ours). The three trajectories were initialized with the same value. The colored contours represent the values of the objective function that is being minimized (least squares).}
    \label{fig:comp}
\end{figure*}

\begin{figure*}
    \centering
    \includegraphics[width=0.8\textwidth]{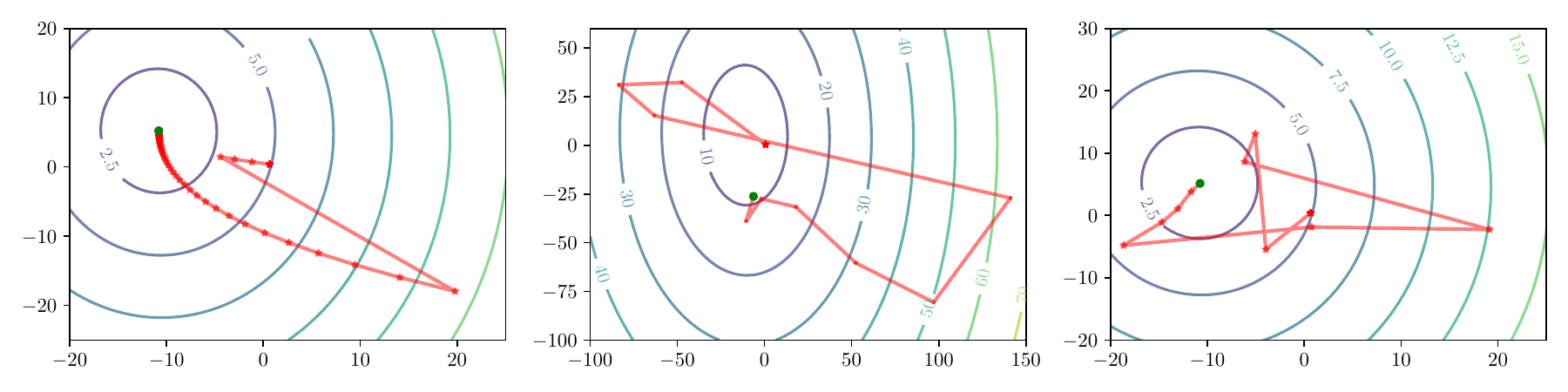}
     \caption{Trajectories of the same test example initialized at the same point but perturbed at the third step with additive noise. The center of the colored contours represents the optimal point, and the final estimate of the optimizers is depicted in green. The standard unrolled optimizer (middle) fails to reach the optimum under additive perturbations, while the constrained one (right), as well as gradient descent (left), succeeds in doing so.}
    \label{fig:comp2}
\end{figure*}

The lack of convergence across the unrolled layers engenders a critical issue despite the superior performance achieved at the last layers. Unrolled networks are brittle to perturbations added to the layers' outputs due to the lack of convergence guarantees. In contrast, standard iterative algorithms with descending guarantees, like GD, are known to be resilient to perturbations as they continue to follow descending directions in the iterations that follow the one where the perturbation occurs (see  \cref{fig:comp2} (left)). This resilience evidently does not extend to their unrolled counterparts since they lack descending guarantees. Revisiting our example, \cref{fig:comp2} (middle) confirms this fact as the unrolled optimizer fails to converge to the same point as in \cref{fig:comp} (middle) after adding perturbations to one of its layers; in fact, the figure shows that the distance to the optimal at the last layer is larger than it is at initialization. While additive perturbations at the layers are artificial and unlikely to occur in practice, this observation indicates that unrolled networks fail to emulate the behavior of the descent algorithms they are designed to mimic. This limits our ability to draw conclusions about the generalizability of these models to out-of-distribution (OOD) problems and raises concerns about implementing them in safety-critical applications. Contrarily, standard iterative algorithms are endowed with theoretical convergence guarantees, thereby eliminating any doubts about their generalizability. 

Previous works have investigated the convergence issue of algorithm unrolling. One approach for providing convergence guarantees, known as safeguarding, has been proposed in \cite{Heaton_Chen_Wang_Yin_2023, shen2021learning, Moeller_2019_ICCV, liu2021investigating}. The main idea of safeguarding is to check whether the estimate made by a certain layer, during inference, satisfies some form of descent criteria. If the criteria are not satisfied, the estimate is discarded and replaced with an estimate of the classic iterative algorithm. Another approach is sequential training, a learning procedure in which the unrolled layers are trained one by one using a tailored loss function for each layer, while the previous layers either continue their training \cite{Ito19} or remain frozen \cite{Sabulal20}.
A different line of work proved the existence of a convergent unrolled network for specific iterative algorithms such as the iterative shrinkage and thresholding algorithm (ISTA) \cite{Chen18theoretical, liu2018alista} and differentiable linearized ADMM \cite{Xie19DL-ADMM}. However, a mechanism that guarantees to find these convergent networks was never provided.
To resolve this issue, \cite{liu2018alista, Abadi15} suggest reducing the size of the search space by learning fewer parameters of the standard algorithm, which limits the network's expressivity. Furthermore, all these theoretical endeavors are tailored for specific optimization algorithms and do not generalize to other algorithms. This leaves many existing and potential unrolled networks without convergence guarantees.

In this paper, we aim to design unrolled algorithms that follow stochastic descending directions to the optimum as in \cref{fig:comp} (right), maintain resilience to additive perturbations as illustrated \cref{fig:comp2} (right) and demonstrate robustness to OOD shifts in the optimization problems they solve. On top of that, our goal is to establish theoretical convergence guarantees for unrolled networks regardless of the underlying iterative algorithms. 
We accomplish this by posing the training problem as a \emph{constrained learning} problem. Within this framework, we enforce each unrolled layer to takes a descent direction, on average. This descent direction needs not be the gradient direction, as shown in \cref{fig:comp} (right), in order to accelerate convergence even for harder landscapes. To complement our endeavor, we theoretically derive convergence guarantees for unrolled networks trained with our descending constraints and prove exponential rates of convergence. We also show that unrolled networks maintain a descending behavior under distribution shifts, with a generalization error that depends on the distance between the original and OOD distributions. In summary, our contributions are as follows:
\begin{itemize}
    \item We develop the use of descending constraints in training unrolled optimizers to provide them with descending properties (Sections \ref{sec:descend_constraint} and \ref{sec:algorithm}).
    \item We rigorously prove that an unrolled network trained under our framework is guaranteed to converge to a near-optimal region infinitely often (Section \ref{sec:theorem}). 
    \item We analyze the generalizability to OOD shifts and show that learning to follow a descending trajectory is transferable under OOD shifts (Section\ref{sec:robustness})
    \item We empirically show the performance and robustness of these unrolled networks in light of two applications: sparse coding (Section \ref{sec:sparse}) and image inpainting (Section \ref{sec:inpainting}).
\end{itemize}

\section{Problem Formulation} \label{sec:PF}

Consider a differentiable function $f$ drawn from the function class $\mathfrak{F}= \{ f(\cdot;{\bf x}) | {\bf x} \in {\cal X} \}$, where ${\cal X}$ is a compact set. The ultimate goal of algorithm unrolling is to learn an $L$-layered neural optimizer $\boldsymbol{\Phi}$ that converges to a stationary point of the function $f(\cdot; {\bf x})$ for any ${\bf x}$. This optimizer $\boldsymbol{\Phi}$ is parameterized by a sequence of parameters ${\bf W} = \{{\bf W}_l \}_{l=1}^L$, each of which resembles the parameters of an iterative optimization algorithm, e.g., proximal gradient descent. The outputs of the unrolled layers form an $L$-long trajectory of estimates, with the ideal goal of 
descending towards a stationary point.

Learning the optimizer $\boldsymbol{\Phi}$ can then be expressed as the bi-level optimization problem,
\begin{align}\label{eq:bi-level}
    \tag{Optimizer}
    \underset{{\bf W}}{\text{min}} \quad &  \mathbb{E} \big[\ell (\boldsymbol{\Phi}({\bf x};{\bf W}) , {\bf y}^*) \big] \\
        \tag{Optimizee}
    {s.t. \quad } &
        {\bf y}^* \in \underset{{\bf y} \in {\cal Y}}{\text{argmin}} \ f({\bf y};{\bf x}), \ \forall {\bf x}\in {\cal X}.
\end{align}
The inner problem is an instantiation of an unconstrained optimization problem, which we refer to as the optimizee, where the goal is to find a minimum of $f(\cdot;{\bf x})$ for a given ${\bf x}$. It, however, needs to be evaluated for all ${\bf x} \in {\cal X}$, which is computationally expensive. The outer problem aims to learn the optimizer $\boldsymbol{\Phi}: {\cal X} \rightarrow {\cal Y}$ by minimizing a suitably-chosen loss function $\ell: {\cal Y} \times {\cal Y} \rightarrow \mathbb{R}$. For instance, the toy example in \cref{fig:comp} considers a function class $\mathfrak{F}$ that contains functions of the form $f({\bf y};  {\bf x}) = \|{\bf x} - {\bf Ay} \|_2$, which is parameterized by ${\bf x} \in {\cal X}$. The optimizee problem is then to find the stationary point of $f$ for each $\bf x$. The optimizer $\boldsymbol{\Phi}$, in this case, is designed by unrolling GD in $L$ layers, each of which is a multilayer perceptron (MLP) with residual connections, i.e., the $l$th layer's output is ${\bf y}_{l} = {\bf y}_{l-1} - \text{MLP}({\bf y}_{l-1}, {\bf x}; {\bf W}_l)$. These MLPs are trained by minimzing the mean-square error (MSE) between the final prediction of $\boldsymbol{\Phi}$ and the optimal ${\bf y}^*$.

Solving the above bi-level optimization problem is not usually straightforward. However, it reduces to an unconstrained problem if the inner problem either has a closed-form solution or can be solved using an iterative algorithm.
For our problems of interest, we assume that there exists an iterative algorithm ${\cal A}$ that finds a (local) minimum ${\bf y}^*$ of the optimizee problem.
This algorithm is the one to be unrolled and is typically a gradient-based method, such as proximal gradient descent.
Given algorithm $\cal A$ and, therefore, the optimal solution to the optimizee,  
the optimizer problem can be then simplified to
\begin{equation}\label{eq:optimizer}
    \begin{split}
        \underset{{\bf W}}{\text{argmin}} \quad &  \mathbb{E} \big[ \ell\big(\boldsymbol{\Phi}({{\bf x}}; {\bf W}), {\bf y}^*\big)\big].
    \end{split}
\end{equation}
This form resembles the empirical risk minimization (ERM) problem used in supervised learning, where a model $\boldsymbol{\Phi}$ learns to map a given input ${\bf x}$ to its label ${\bf y}^*$ with little to no restrictions on the intermediate layers of $\boldsymbol{\Phi}$. 

It is pertinent to recall that the model $\boldsymbol{\Phi}$ in \eqref{eq:optimizer} is designed to mimic an iterative descending algorithm, where the output of each layer needs to take a step closer to the optimum. This behavior is not mandated by the loss function in \eqref{eq:optimizer} since there is no \emph{regularization} on the outputs of the optimizer's intermediate layers. The absence of regularization leads to the behavior in Figure \ref{fig:comp} (middle), where the unrolled optimizer hits the optimum at the last layer while the trajectory to the optimum is not what we expect from a standard descending algorithm. Since standard unrolling does not necessarily generate descending trajectories, they lack the convergence guarantees, inherent in standard optimization algorithms. This raises concerns about the generalizability of the unrolled optimizer, trained over finite samples from the data distribution, to unseen in-distribution functions $f(\cdot; {\bf x}) \in \mathfrak{F}$.

Yet another issue regarding the stability of unrolled optimizers becomes apparent when encountering additive perturbations \cite{hadou2023space, hadouspace, gama20}. Due to convergence guarantees, standard algorithms are guaranteed to continue descending even after perturbing their trajectories with additive noise. Reverting to our example, Figure \ref{fig:comp2} (middle) is evidence that this feature is not guaranteed for standard unrolled networks, which raises concerns about their stability.

To put it briefly, lacking convergence guarantees prompts concerns pertaining to generalizability and vulnerability to perturbations. In the following section, we tackle this issue by \emph{imposing descending constraints} over the intermediate layers. Under these constraints, we can think of the unrolled optimizers as stochastic descent algorithms, for which we can provide convergence guarantees.

\begin{remark}\label{rem:unsupervised}
While we only train unrolled optimizers in a supervised manner in this paper (c.f. \eqref{eq:optimizer}), our problem formulation and proposed method extend to the unsupervised case, where the optimal ${\bf y}^*$ is expensive to obtain. The learning problem then can be seen as directly minimizing $f$, i.e.,
\begin{equation}\label{eq:unsupervised}
    \begin{split}
        \underset{{\bf W}}{\text{argmin}} \quad &   \mathbb{E} \big[f\big(\boldsymbol{\Phi}({{\bf x}}; {\bf W})\big) \big] .
    \end{split}
\end{equation}
This approach is commonly used in the \emph{learning-to-learn} literature, e.g.,\cite{hadou2023stochastic, Ravi2016OptimizationAA, Andrychowicz16, liu2022optimization}. 
\end{remark}
\section{Constrained Unrolled Optimizers}\label{sec:proposal}
To provide unrolled algorithms with convergence guarantees, we propose the use of descending constraints in training the unrolled optimizers relying on constrained learning theory (CLT) \cite{chamon2022constrained}. The posed constrained learning problem \eqref{eq:constrainedUO}, while it may initially appear complex, is akin to a regularized ERM problem where the regularization parameter is a (dual) optimization variable to be learned. The key advantage of our approach over a typical regularized problem stems from the sensitivity interpretation of the dual variables \cite{boyd2004convex}, which easily discerns constraints that are harder to satisfy. This, in turn, grants us informed insights on which constraints to relax, if possible, during the design phase. In this section, we delineate our method to design these constrained unrolled optimizers by providing a detailed explanation of i) our choice of the descending constraints to employ, ii) the algorithm that solves the constrained problem and provides a probably near-optimal near-feasible solution, and iii) the convergence guarantees bestowed upon these optimizers.

\subsection{Descending Constraints}\label{sec:descend_constraint}
To force the unrolled optimizer $\boldsymbol{\Phi}$ to produce descending trajectories across its layers, 
we consider two different forms of descending constraints. The first ensures that the gradient norm of the objective function $f$ decreases over the layers on average, i.e., for each layer $l = 1, \dots, L$, we have
\begin{equation}\label{eq:gradConst}
    \tag{CI}
    \begin{split}
        \mathbb{E} \big[ \| {\nabla} f({\bf y}_{l};{\bf x})\|_2  -  (1-\epsilon) \ \|  {\nabla} f({\bf y}_{l-1};{\bf x}) \|_2 \big] \leq 0,
    \end{split}
\end{equation}
where ${\bf y}_l$ is the $l$-th unrolled layer's output, ${\bf y}_L = \boldsymbol{\Phi}({\bf x};{\bf W})$ is the final estimate, and $\epsilon \in (0,1)$ is a design parameter. The second forces the distance to the optimal to shrink over the layers, that is,
\begin{equation}\label{eq:distConst}
    \tag{CII}
    \begin{split}
        \mathbb{E} \big[ \| {\bf y}_{l} - {\bf y}^*\|_2 \big.- \big. (1-\epsilon) \ \|  {\bf y}_{l-1} - {\bf y}^* \|_2
        \big] \leq 0.
    \end{split}
\end{equation}
The initial value ${\bf y}_0$ is drawn from a Gaussian distribution ${\cal N}(\mu_0, \sigma_0^2{\bf I})$. The expectation in the two constraints is evaluated over the joint distribution of the input $\bf x$ and the initialization ${\bf y}_0$. We also add noise vector ${\bf n}_l$, sampled from a Gaussian distribution ${\cal N}(0, \sigma^2_l{\bf I}) =: {\cal N}_l$, to the input of layer $l$--during training--to ensure that the output ${\bf y}_l$ is independent of the trajectory the network follows up till layer $l$. The noise variance $\sigma_l^2$ decreases over the layers to allow the optimizer to converge.

Having these two constraints in place, we re-write \eqref{eq:optimizer} as the following constrained learning problem
\begin{equation}\label{eq:constrainedUO}
    \tag{CO}
    \begin{split}
        P^* = \min_{{\bf W}} \quad & \mathbb{E} \left[\ell(\boldsymbol{\Phi}({\bf x};{\bf W}), {\bf y}^*) \right]\\
        {s. t.} \quad \ &
        \mathbb{E}_{{\cal N}_l} \big[{\cal C}({\bf y}_l, {\bf y}_{l-1})\big] \leq 0, \ \forall l\leq L, \\
        & {\bf y}_l = \phi({\bf y}_{l-1} + {\bf n}_l;{\bf W}_l), \ 0<l\leq L,
    \end{split}
\end{equation}
where ${\cal C}(.,.)$ refers to the loss function in either \eqref{eq:gradConst} or \eqref{eq:distConst}, and $\phi(\cdot;{\bf W}_l)$ represents a layer in $\boldsymbol{\Phi}$ parameterized by ${\bf W}_l$. The second set of constraints in \eqref{eq:constrainedUO} should be perceived as implicit constraints that are imposed by default in the structure of the unrolled optimizer $\boldsymbol{\Phi}$ and, therefore, omitted in the following analysis.

It is worth noting that the choice between \eqref{eq:gradConst} or \eqref{eq:distConst} is arbitrary. The former requires the objective function $f$ to be differentiable while the latter depends on the knowledge of the optimal ${\bf y}^*$. Yet \eqref{eq:gradConst} can be extended to non-differentiable functions by replacing the gradient with a subgradient of $f$. However, \eqref{eq:gradConst} could be more challenging to satisfy if the optimizee problem is either constrained or non-convex. This is true since it could push the estimates to a non-feasible solution in the first instance or another local minimum than ${\bf y}^*$ in the second one. In the case of unsupervised training (as in \cref{rem:unsupervised}), the constraints \eqref{eq:distConst} add extra computational complexity since it requires calculating the optimal solution ${\bf y}^*$ that is no longer required in evaluating the training loss. Other descending constraints could also be used, such as decreasing the value of the function $f$ directly.

\begin{remark}
The noise added to each layer's output is present only during training and is omitted during execution. On one hand, it ensures that each layer takes a step independently of the path followed by the network in previous layers. On the other hand, it compels the unrolled network to explore the optimization landscape during training, which empirically helps the model's generalizability when accompanied with the proposed descending constraints. It is pertinent to notice that our theoretical analysis in the subsequent sections remains valid when the noise is entirely omitted.   
\end{remark}

\subsection{Probably and Approximately Correct Solution}\label{sec:algorithm}
To solve \eqref{eq:constrainedUO}, we construct the dual problem to leverage CLT \cite[Theorem 1]{chamon2022constrained}.
The dual problem of \eqref{eq:constrainedUO} is equivalent to finding the saddle point of the Lagrangian function 
\begin{equation}\label{eq:statlagrang}
    \begin{split}
        &{\cal L}({\bf W}, {\boldsymbol \lambda}) =  \mathbb{E} \left[\ell(\boldsymbol{\Phi}({\bf x};{\bf W}) , {\bf y}^*) \right] 
        + 
        \sum_{l=1}^L {\lambda}_{l}
        \mathbb{E} \big[ {\cal C}({\bf y}_l, {\bf y}_{l-1})    \Big],
    \end{split}
\end{equation}
where ${\boldsymbol \lambda} \in {\mathbb{R}^{L}_+}$ is a vector collecting the dual variables $\lambda_{l}$. Since evaluating the expectation over an unknown distribution is not attainable, we resort to using an \emph{empirical} Lagrangian function
\begin{equation}\label{eq:emplagrang}
    \begin{split}
        & \widehat{\cal L}({\bf W}, {\boldsymbol \lambda}) =  \widehat{\mathbb{E}} [\ell(\boldsymbol{\Phi}({\bf x}; {\bf W}), {\bf y}^*) ] 
        + 
        \sum_{l=1}^L  {\lambda}_{l}
        \widehat{\mathbb{E}} \big[ {\cal C}({\bf y}_l, {\bf y}_{l-1}) \big],
    \end{split}
\end{equation}
where $\widehat{\mathbb{E}}$ denotes the sample mean evaluated over $N$ realizations $\{({\bf x}_{n}, {\bf y}_{n}^*)\}_{n=1}^{N}$. We then define the empirical dual problem as 
\begin{equation}\label{eq:dual}
    \tag{DO}
    \begin{split}
        \widehat{D}^* = \max_{{\boldsymbol \lambda} \in \mathbb{R}^{L}_+} \ \min_{{\bf W}} \ \widehat{\cal L}({\bf W}, {\boldsymbol \lambda}).
    \end{split}
\end{equation}
Solving \eqref{eq:dual} and finding the saddle point of $\widehat{\cal L}$ is doable by alternating between minimizing with respect to ${\bf W}$ for a fixed $\boldsymbol{\lambda}$ and then maximizing the latter, as described in Algorithm \ref{alg:PD}. The convergence of this algorithm has been proved in \cite[Theorem 2]{chamon2022constrained}.
\begin{algorithm}[!t]
\caption{Primal-Dual Training Algorithm for \eqref{eq:constrainedUO}}\label{alg:PD}
\begin{algorithmic}[1]
 \STATE {\bfseries Input}: Dataset $\{({\bf x}_{n}, {\bf y}_{n}^*) \sim D_x\}_{n=1}^{N}$.  
 \STATE Initialize ${\bf W} =\{{\bf W}_l\}_{l=1}^L$ and ${\boldsymbol \lambda} = \{ {\boldsymbol \lambda}_{l} \}_{l=1}^L$.
 \FOR{each epoch}
    \FOR{each batch}
    \STATE Compute $\widehat{\cal L}({\bf W}, {\boldsymbol \lambda})$ as in \eqref{eq:emplagrang}.
    \STATE Update the primal variable:
    \begin{equation}
        {\bf W} \gets [{\bf W} - \mu_{w} \nabla_{{\bf W}} \widehat{\cal L}({\bf W}, {\boldsymbol \lambda})].
    \end{equation}
    \ENDFOR
\STATE Update the dual variable:
\begin{equation}
    \begin{split}
        {\boldsymbol \lambda} & \gets [{\boldsymbol \lambda} + \mu_{\lambda} \nabla_{{\boldsymbol \lambda}} \widehat{\cal L}({\bf W}, {\boldsymbol \lambda})]_+.
    \end{split}
\end{equation}
\ENDFOR
 \STATE {\bfseries Return:} ${\bf W}^* \gets {\bf W}.$
\end{algorithmic}
\end{algorithm}

Nevertheless, Algorithm \ref{alg:PD} only finds a solution to \eqref{eq:dual}, which is not equivalent to the primal problem \eqref{eq:constrainedUO} that we aim to solve. The difference between the two problems is due to an \emph{approximation} gap and \emph{estimation} gap induced by replacing the statistical expectations with empirical ones. A characterization of this gap is provided by CLT under the following assumptions:
\begin{assumption}\label{A1}
The functions $\ell(\cdot, {\bf y})$ and ${\cal C}(\cdot, {\bf y})$ are $M$-Lipschitz continuous for all ${\bf y}$ and bounded. The objective function $f(\cdot; {\bf x})$ is also $M$-Lipschitz continuous for all ${\bf x}$.
\end{assumption}
\begin{assumption}\label{A2}
Let $\phi_l {\scriptscriptstyle\circ} \dots {\scriptscriptstyle\circ} \phi_1 \in {\cal P}_l$ be a composition of $l$ unrolled layers, each of which is parameterized by ${\bf W}_{i}, i \leq l$, and $\overline{\cal P}_l = \overline{conv}({\cal P}_l)$ be the convex hull of ${\cal P}_l$.
Then, for each $\overline\phi_l {\scriptscriptstyle\circ} \dots {\scriptscriptstyle\circ} \overline\phi_1 \in \overline{\cal P}$ and $\nu > 0$, there exists ${\bf W}_{1:l}$ such that $\mathbb{E} \left[ |\phi_l {\scriptscriptstyle\circ} \dots {\scriptscriptstyle\circ} \phi_1({\bf x}; {\bf W}_{1:l}) - \overline\phi_l {\scriptscriptstyle\circ} \dots {\scriptscriptstyle\circ} \overline\phi_1({\bf x})| \right] \leq \nu$. 
\end{assumption}
\begin{assumption}\label{A3}
The set $\cal Y$ is compact, the conditional distribution $p({\bf x}|{\bf y}^*)$ is non-atomic for all ${\bf y}^*$, and
the functions ${\bf y} \to \ell(\boldsymbol{\Phi}(\cdot), {\bf y}) p(\cdot|{\bf y})$ and ${\bf y} \to {\cal C}(\boldsymbol{\Phi}(\cdot), {\bf y}) p(\cdot|{\bf y})$ are uniform continuous in the total variation topology for all $\boldsymbol{\Phi} \in \overline{\cal P}_l, l\leq L$.
\end{assumption}
Assumptions \ref{A2} and \ref{A3} aim to establish a connection between \eqref{eq:constrainedUO} and its functional counterpart that learns a composition of $L$ functions $ \overline\phi_L {\scriptscriptstyle\circ} \dots {\scriptscriptstyle\circ} \overline\phi_1 \in \overline{\cal P}_L $. Assumption \ref{A3} asserts a zero duality gap between the functional problem and its dual while Assumption \ref{A2} ensures that the parameterization ${\bf W}$ is rich enough to approximate the functional space $\overline{\cal P}_L$. Through these assumptions, CLT analyzes the change to the duality gap of the functional problem when it is  approximated by \eqref{eq:constrainedUO}. On the other hand, CLT characterizes the estimation gap through the sample complexity indicated by Assumption \ref{A4}. Moreover, CLT requires the problem to be strictly feasible, as described in Assumption \ref{A5}.
\begin{assumption}\label{A4}
There exists $\zeta(N, \delta)\geq 0$ that is monotonically decreasing with $N$, for which it holds with probability $1-\delta$ that 
$|\mathbb{E}[\ell(\boldsymbol{\Phi}({\bf x}; {\bf W}), {\bf y}^*)] - \widehat{\mathbb{E}}[\ell(\boldsymbol{\Phi}({\bf x}; {\bf W}), {\bf y}^*)]| \leq \zeta(N, \delta)$ and
$|\mathbb{E}[{\cal C}({\bf y}_l, {\bf y}_{l-1})] - \widehat{\mathbb{E}}[{\cal C}({\bf y}_l, {\bf y}_{l-1})]| \leq \zeta(N, \delta)$
for all $l$.
\end{assumption}
\begin{assumption}\label{A5}
There exists $\boldsymbol{\Phi} \in {\cal P}_L$ that is strictly feasible, i.e., $\mathbb{E} \big[{\cal C}({\bf y}_l, {\bf y}_{l-1})\big] \leq -M\nu - \xi$ and $\widehat{\mathbb{E}} \big[{\cal C}({\bf y}_l, {\bf y}_{l-1})\big] \leq -\xi$ for all $l\leq L$, with $M$ and $\nu$ as in Assumptions \ref{A1} and \ref{A2} and $\xi > 0$.
\end{assumption}

These assumptions are easily achievable in practice. The Lipschitz continuity in Assumption \ref{A1} holds for a wide range of loss functions, including $\ell_1$ and $\ell_2$ norms. Modern neural networks satisfy the universal approximation theorem (Assumption \ref{A2}), which demands that the family of neural networks that we consider to be sufficiently rich to approximate the update function $\phi$ up to a factor $\nu$. Assumption \ref{A3} requires the output of the neural network to be bounded, which can be ensured using a compact input set and bounded learnable parameterization. Additionally, it requires the conditional probabilities to be nonatomic, a condition that can be met by augmenting the dataset with solutions of varying precision for each instance of the optimizee problem. Assumption \ref{A4} imposes a mild assumption on the sample complexity of the models, allowing us replace the statistical expectation with the empirical average. The strict feasibility condition in Assumption \ref{A5} can also be attained by appropriately adjusting the design parameter $\epsilon$.

CLT then asserts that a stationary point of \eqref{eq:dual} provides a probably, approximately correct solution to \eqref{eq:constrainedUO}. 
\begin{theorem}[CLT \cite{chamon2022constrained}]\label{lem:CL}
Let $({\bf W}^*, \boldsymbol{\lambda}^*)$ be a stationary point of \eqref{eq:dual}. Under Assumptions \ref{A1}- \ref{A5}, it holds, for some constant $\rho$, that 
\begin{equation}\label{eq:nearOptimal}
    |P^* - \widehat{D}^*| \leq M \nu + \rho \ \zeta(N, \delta), \ \text{and}
\end{equation} 
\begin{equation}\label{eq:constraints}
    \begin{split}
         & \mathbb{E} \big[ {\cal C}({\bf y}_l, {\bf y}_{l-1}) \big]  \leq \zeta(N, \delta), \quad  \forall l,
    \end{split}
\end{equation}
with probability $1-\delta$ each and with $\rho \geq \max \{ \| \boldsymbol{\lambda}^*\|, \| \overline{\boldsymbol{\lambda}}^* \| \}$, where $\overline{\boldsymbol{\lambda}}^* = \textup{argmax}_{{\boldsymbol \lambda}} \ \min_{{\bf W}} \ {\cal L}({\bf W}, {\boldsymbol \lambda})$.
\end{theorem}
As per Theorem \ref{lem:CL}, the solution ${\bf W}^*$ is near-optimal and near-feasible. Moreover, the duality gap depends on a smoothness constant $M$, the richness parameter $\nu$, the sample complexity $\zeta(N, \delta)$, and the sensitivity to the constraints embodied by the dual variables $\boldsymbol{\lambda}^*$.

Theorem \ref{lem:CL} also implies that an optimizer trained via Algorithm \ref{alg:PD} satisfies the descending constraints at each layer up to a fixed margin with a probability of $1-\delta$. This margin solely depends on the sample complexity, which can be kept small by increasing the number of samples $N$. This result postulates that at each layer, the unrolled optimizer with a high probability takes a step closer to an optimal point of the optimizee. As a consequence, the trained optimizer can be interpreted as a \textit{stochastic} descent algorithm.


\subsection{Convergence Guarantees}\label{sec:theorem}
The above result does not directly guarantee that the sequence $\{{\bf y}_l \}_{l=1}^L$ obtained at the outputs of the unrolled layers \textit{converges} to a stationary point of the function $f({\bf y}; \cdot)$. The probably, approximately correct solution we have attained only satisfies the descending constraints with a probability $1-\delta$. Even for small values of $\delta$, the probability that all the constraints are satisfied together, which is $(1-\delta)^L$, is exponentially decreasing with the number of layers $L$. Therefore, a rigorous analysis is still required to affirm the convergence guarantees of unrolled optimizers. In Theorem \ref{thm:convergence}, we confirm that the unrolled optimizer trained under constraints \eqref{eq:gradConst} indeed has convergence guarantees.

\begin{theorem}\label{thm:convergence} Given a near-optimal solution of \eqref{eq:constrainedUO} under constraints \eqref{eq:gradConst}, which satisfies \eqref{eq:constraints} with a probability $1-\delta$ and generates a sequence of random variables $\{ {\bf y}_l | l \geq 0 \}$. Then, under Assumption \ref{A1}, it holds that 
    \begin{equation}
    \lim_{l \rightarrow \infty} \mathbb{E}  \Big[ \min_{k\leq l} \|{\nabla} f({\bf y}_{k}; {\bf x})\| \Big] \leq \frac{1}{\epsilon} \left(\zeta(N, \delta) + \frac{\delta M}{1-\delta}\right) \ \ a.s.
    \end{equation}
with $\zeta(N, \delta)$ as described in Assumption \ref{A4}.
\end{theorem}
\begin{proof}
    The proof is relegated to Appendix \ref{app:thm1}.
\end{proof}
 Theorem \ref{thm:convergence} states that the sequence $\{ {\bf y}_l\}_l$ infinitely often visits a region around the optimal where the expected value of the gradient norm hits a value under $\frac{1}{\epsilon} \big(\zeta(N, \delta) + \frac{\delta M}{1-\delta}\big)$. This near-optimal region is determined by the sample complexity of $\boldsymbol{\Phi}$, the Lipschitz constant $M$, and a design parameter $\epsilon$ of the constraints. The same analysis can be conducted for the case when the constraints \eqref{eq:distConst} is considered leading to similar guarantees.
 




The convergence rate of the unrolled network trained according to \eqref{eq:constrainedUO} is derived in \cref{thm:rate}. As per this lemma, the expected value of the gradient norm drops exponentially with the number of unrolled layers $L$. This result is explicable through the recursive nature of the imposed constraints. 
\begin{lemma}\label{thm:rate}
For a trained unrolled optimizer ${\bf W}^*$ that satisfies Theorem \ref{lem:CL}, the gradient norm achieved after $L$ layers satisfies
\begin{equation}
\begin{split}
     \mathbb{E} \big[ {\|{\nabla} f({\bf y}_{L})\|} \big] 
    \leq (1-\delta)^L(&1-{\epsilon})^L \ \mathbb{E} {\|  {\nabla} f({\bf y}_{0})\|} \\
    & + \frac{1}{\epsilon} \left( \zeta(N, \delta)
      + \frac{\delta M}{1-\delta} \right).
\end{split}
\end{equation}
\end{lemma}
\begin{proof}
    The proof can be found in Appendix \ref{app:rate}.
\end{proof}

\begin{remark} \cref{thm:convergence} provides convergence guarantees for unrolled networks that parallel those of stochastic gradient-based algorithms. The similarity lies in the fact that that both algorithms follow stochastic descending directions and keep visiting a small area around the optimum infintely often.
Although excursions away from this small area are possible and can
be arbitrarily large, existing studies, such as \cite{Eksin12}, show that the maximum value reached in these excursions is bounded above and adheres to an exponential probability bound. These findings can be directly expanded to unrolled networks that satisfy the descending constraints.
\end{remark}

\subsection{Robustness to Out-of-Distribution Shifts}\label{sec:robustness}

The significance of the descending constraints arises from the fact that the unrolled networks learn how to find the solution by following a descending direction, rather than learning the solution itself. The latter strategy fails against OOD shifts where the input $\bf x$ is drawn from a distribution $D_x'$ different from the original distribution $D_x$ used during training. These distribution shifts represent a change in the optimizee problem, resulting in a change in the distribution of the optimal solutions, which unrolled network cannot directly adapt to. The descending constraints, which are satisfied in expectation under the distribution $D_x$, can provide the unrolled networks with robustness against such distribution shifts, provided they remain satisfied under the new distribution $D_x'$. The following corollaries, which hold under Assumption \ref{A6}, prove that the descending nature of the layers can be transferred under distribution shifts and offer an upper bound for the generalization error. While these corollaries specifically handle the first constraints \eqref{eq:gradConst},the same results can be readily expanded to the second set of constraints \eqref{eq:distConst}.

\begin{assumption}\label{A6}
There exists a non-negative symmetric distance $d$ between the input distribution $D_x$ and the OOD distribution $D_x'$ such that
\begin{equation*}
   \Big| \mathbb{E}_{D_x} \big[ {\|{\nabla} f({\bf y}; {\bf x})\|} \Big]  - \mathbb{E}_{D_x'} \big[ {\|{\nabla} f({\bf y}; {\bf x})\|} \big] \Big| \leq M d(D_x, D_x')
\end{equation*}    
uniformly over ${\bf y}$ with $M$ being a Lipschitz constant.
\end{assumption}

This assumption is easily attainable. For instance, the Wasserstein distance $W_1(D_x, D_x')$ satisfies this assumption and provides a bounded distance between the two distributions $D_x$ and $D_x'$. Given an unrolled network trained on the distribution $D_x$, the following corollary evaluates the descending constraints when the network is executed under the distribution $D_x'$.

\begin{corollary}\label{cor:OOD}
Under Assumption \ref{A6}, an unrolled network ${\bf W}$ trained via \cref{alg:PD} over a distribution $D_x$ generates when executed on an OOD distribution $D_x'$ a sequence of layers' outputs $\{{\bf y}_l\}_l$  that satisfies
\begin{equation}
\begin{split}
    \mathbb{E}_{D_x'} \big[ \| {\nabla} f({\bf y}_{l};{\bf x})\|_2 \ - & \ (1-\epsilon) \ \|  {\nabla} f({\bf y}_{l-1};{\bf x}) \|_2 \big] \\
    & \leq \zeta(N, \delta) + 2M d(D_x, D_x'),
\end{split}
\end{equation}   
with probability $1-\delta$ each.
\end{corollary}

\begin{proof}
    The proof can be found in Appendix \ref{app:OOD}.
\end{proof}

\cref{cor:OOD} indicates that the descending constraints are satisfied under an OOD distribution $D_x'$ up to a generalization error that depends on the distance between the two distributions $D_x$ and $D_x'$. 
Even under distribution shifts, the unrolled network still finds (noisy) descending directions toward the optimum, despite not being trained on the solutions of the new optimizee problems. This capability is attributed to the proposed descending constraints and is unique to the unrolled networks trained with such constraints. 

Since the descending constraints are satisfied with a marginal error and with probability $1-\delta$, convergence guarantees for the unrolled network can be derived under the new distribution similarly to \cref{thm:convergence}. The following corollary summarizes this result.

\begin{corollary} 
The sequence of layers' outputs $\{{\bf y}_l\}_l$ generated according to \cref{cor:OOD} also satisfies
\begin{equation}
\begin{split}
    \lim_{l \rightarrow \infty} \mathbb{E}_{D_x'}  \Big[ &  \min_{k\leq l} \|{\nabla} f({\bf y}_{k}; {\bf x})\| \Big] \\
    & \leq \frac{1}{\epsilon} \Big(\zeta(N, \delta) + 2Md(D_x, D_x') +
     \frac{\delta M}{1-\delta}\Big).
\end{split}
\end{equation}    
\end{corollary}

This corollary follows directly from \cref{thm:convergence}. 
It states that, under distribution shifts, the unrolled network converges eventually to a region around the optimal solution where the gradient norm is below a factor that includes the distance between the two distributions. The smaller this distance, the smaller the near-optimal region to which the network converges.

In the following two sections, we consider two use cases, namely sparse coding and inverse problems, where algorithm unrolling achieves superior performance. We aim to show the impact of our proposed descending constraints on the performance and robustness of these two use cases.

\section{Case Study I: LISTA for Sparse Coding}\label{sec:sparse}
In our first case study, we tackle the sparse coding problem, which is of enduring interest in signal processing. In this section, we provide a brief description of the problem and its iterative solution. We also show the mechanism that has been previously proposed to unroll this iterative algorithm and then provide discussions, supported by numerical simulations, to evaluate the proposed descending constraints.
\subsection{Problem Formulation} 
Sparse coding refers to the problem of finding a sparse representation of an input signal ${\bf x} \in \mathbb{R}^p$ using an overcomplete dictionary ${\bf D}\in \mathbb{R}^{p\times d}$ with $d>p$. The goal is to find a sparse code ${\bf y}\in \mathbb{R}^d$ that satisfies ${\bf x} \approx \bf{Dy}$.
This optimizee problem can be cast as a LASSO regression problem
\begin{equation}\label{eq:sparse_coding}
    \begin{split}
        \min_{{\bf y}} \quad & f_{sp}({\bf y}; {\bf x}) := \frac{1}{2} \|{\bf x} - {\bf Dy}\|_2^2 + \alpha \|{\bf y}\|_1,
    \end{split}
\end{equation}
where $\alpha > 0$ is a regularization parameter. The $\ell_1$-norm is used as a regularization term to control the sparsity of the solution. 

A widely-used method for solving \eqref{eq:sparse_coding} is the iterative shrinkage and thresholding algorithm (ISTA) \cite{ISTA2004, beck_fast_2009}. At each iteration $k$, ISTA updates the $k$-th iteration of the solution, denoted by ${\bf y}_k$, according to the following rule:
\begin{equation}\label{eq:ISTA}
    {\bf y}_{k} = {\cal S}_{\alpha/\nu} \left({\bf y}_{k-1} - \frac{1}{\nu} {\bf D}^\top ({\bf Dy}_{k-1} -{\bf x})\right),
\end{equation}
where $\nu$ is a parameter whose value is larger than the largest eigenvalue of ${\bf D}^\top {\bf D}$, and ${\cal S}_{\alpha/\nu}$ is the soft-thresholding operator,
\begin{equation}
    {\cal S}_{\alpha/\nu}({\bf y}) = \text{sign}({\bf y}) \cdot \max \{ |{\bf y}|-{\alpha/\nu}, 0\},
\end{equation}
which is evaluated elementwise.
The update rule in \eqref{eq:ISTA} performs a proximal gradient descent step, which is equivalent to a gradient step in the direction of $-\nabla \| {\bf x} - {\bf Dy}\|_2^2/2$ followed by a projection onto the $\ell_1$-norm ball. This projection ensures that the new value ${\bf y}_k$ is within the feasible set of the optimizee problem.

\begin{figure*}
    \centering
    \includegraphics[width=0.7\textwidth]{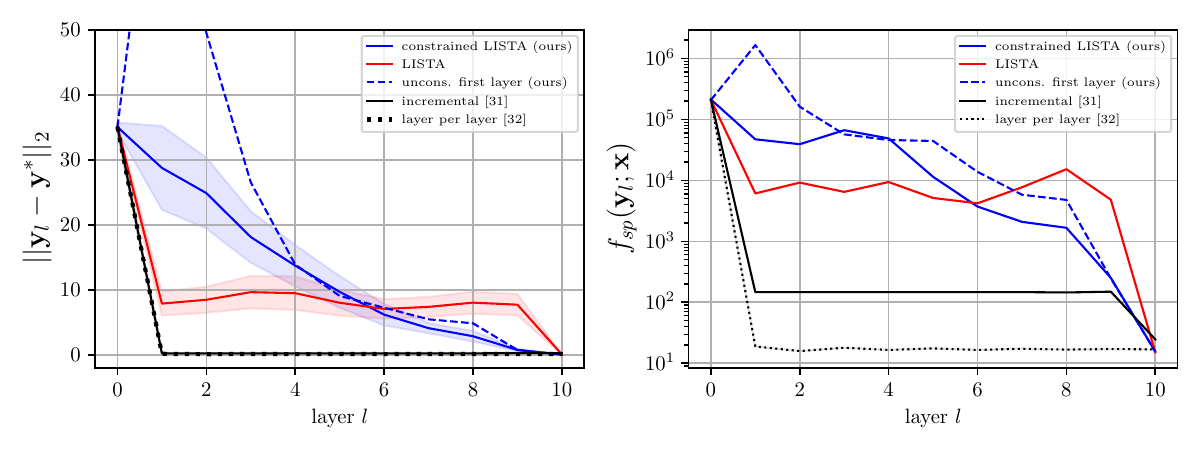}
    \caption{Distance to the optimal solution ${\bf y}^*$ and the value of the objective function $f_{sp}({\bf y}_l; {\bf x})$ across the ten unrolled layers of constrained LISTA (blue), LISTA (red), incremental training (black), and layer-per-layer training (black dots).  
    Constrained LISTA makes gradual progressions toward the optimum, unlike the other approaches, due to the implemented descending constraints during training.}
    \label{fig:distance_LISTA}
\end{figure*}

\subsection{LISTA: Unrolled Solution} 
The seminal work in \cite{gregor_learning_2010} has unrolled the update rule of ISTA and introduced Learnable ISTA (LISTA) for sparse coding contexts. This work has re-identified the iteration in \eqref{eq:ISTA}
as a linear mapping of ${\bf y}$ followed by a nonlinear activation function (i.e., a soft-thresholding function), i.e.,
\begin{equation}\label{eq:LISTA}
    {\bf y}_l = {\cal S}_{\boldsymbol{\beta}^l} \left( {\bf D}^l_u {\bf x} + {\bf D}^l_e {\bf y}_{l-1} \right),
\end{equation}
with ${\bf D}^l_u \in \mathbb{R}^{d \times p}, {\bf D}_e^l \in \mathbb{R}^{d \times d}$ and $\boldsymbol{\beta}^l \in \mathbb{R}^d$ being learnable parameters. 
Cascading $L$ of these layers in an unrolled architecture is then equivalent to executing ISTA  for $L$ iterations.
This unrolled architecture has two notable features. First, ${\bf D}^l_u$ represents a residual (i.e., skip) connection from the input ${\bf x}$ to layer $l$, which is reminiscent of ResNets. Second, the network employs a parametric nonlinear activation function, which has been a recent trend in designing neural networks \cite{chung16, kiliccarslan2021rsigelu, varshney2021optimizing}.

As can be observed by comparing \eqref{eq:LISTA} to \eqref{eq:ISTA}, LISTA has decoupled its parameters by making the following substitutions: ${\bf D}_u = (1/L){\bf D}^\top$ and ${\bf D}_e = {\bf I} - (1/L){\bf D}^\top {\bf D}$. 
However, the decoupling of ${\bf D}_u$ and ${\bf D}_e$ was proven to have a negative impact on the convergence of LISTA \cite{Chen18theoretical}.

\subsection{Constrained LISTA}
In our approach, we adopt the same structure of LISTA in \eqref{eq:LISTA}, but we perform a constrained learning-based training procedure according to \eqref{eq:constrainedUO}. 
As in \cite{gregor_learning_2010}, the unrolled network is trained by minimizing the mean square error (MSE);
\begin{equation}\label{eq:cLISTA_training}
    \begin{split}
        \min_{{\bf W}} \quad &  \mathbb{E} \big[\|\boldsymbol{\Phi}({\bf x};{\bf W}) - {\bf y}^*\|_2^2 \big].
    \quad \text{s.t.} \quad \text{\eqref{eq:distConst}}.
    \end{split}
\end{equation}
Relying on the constraints \eqref{eq:distConst} is a deliberate choice that fits the nature of the optimizee problem \eqref{eq:sparse_coding}. The optimizee is originally a constrained problem, where the goal is to find a point within the $\ell_1$-norm ball that minimizes the objective function. Following the gradients of the $\ell_2$-norm would not guarantee convergence to a feasible solution unless we add explicit $\ell_1$-norm constraints on the outputs of the unrolled layers.

\subsection{Numerical Results and Discussions}
In our experiments, we evaluate the performance of the proposed constrained-learning approach on grayscale CIFAR10 dataset. Each image has $32 \times 32$ pixels and is flattened into a 1024-dimensional vector.
To fully identify the optimizee problem, using $d=1228$, we construct the dictionary $\bf D$ as a random matrix with each element independently drawn from a Gaussian distribution ${\cal N}(0, 1)$ and we set $\alpha = 0.5$. Under this set-up, we evaluate the sparse vector ${\bf y}^* \in \mathbb{R}^{1228}$ for each image ${\bf x} \in \mathbb{R}^{1024}$ by executing ISTA for $10$ thousands iterations and construct a dataset $\{ ({\bf x}_i, {\bf y}_i^*)\}$ to train the unrolled LISTA. The dataset consists of $32 \times 10^3$ training, $8\times 10^3$ validation and $10^4$ test examples.

The unrolled LISTA consists of $L=10$ layers. We take the initial estimate ${\bf y}_0$ to be a random vector drawn from ${\cal N}(0, {\bf I})$, which is then fed to the unrolled network along with the input images. During training, the layers' outputs are contaminated with additive Gaussian noise with zero mean and a variance of $\sigma_l^2 = \hat{\sigma}^2g_l^2$, where $g_l$ is the gradient at layer $l$ and $\hat{\sigma} = 1$. We train the unrolled model for $30$ epochs using ADAM with a learning rate $\mu_w = 10^{-5}$ and a dual learning rate $\mu_\lambda = 10^{-3}$. We also set the constraint parameter $\epsilon$ to $0.05$.\footnote{The code is available at: \href{https://github.com/SMRhadou/RobustUnrolling}{https://github.com/SMRhadou/RobustUnrolling}.}

\begin{figure*}[t]
    \centering
    \includegraphics[width=\textwidth]{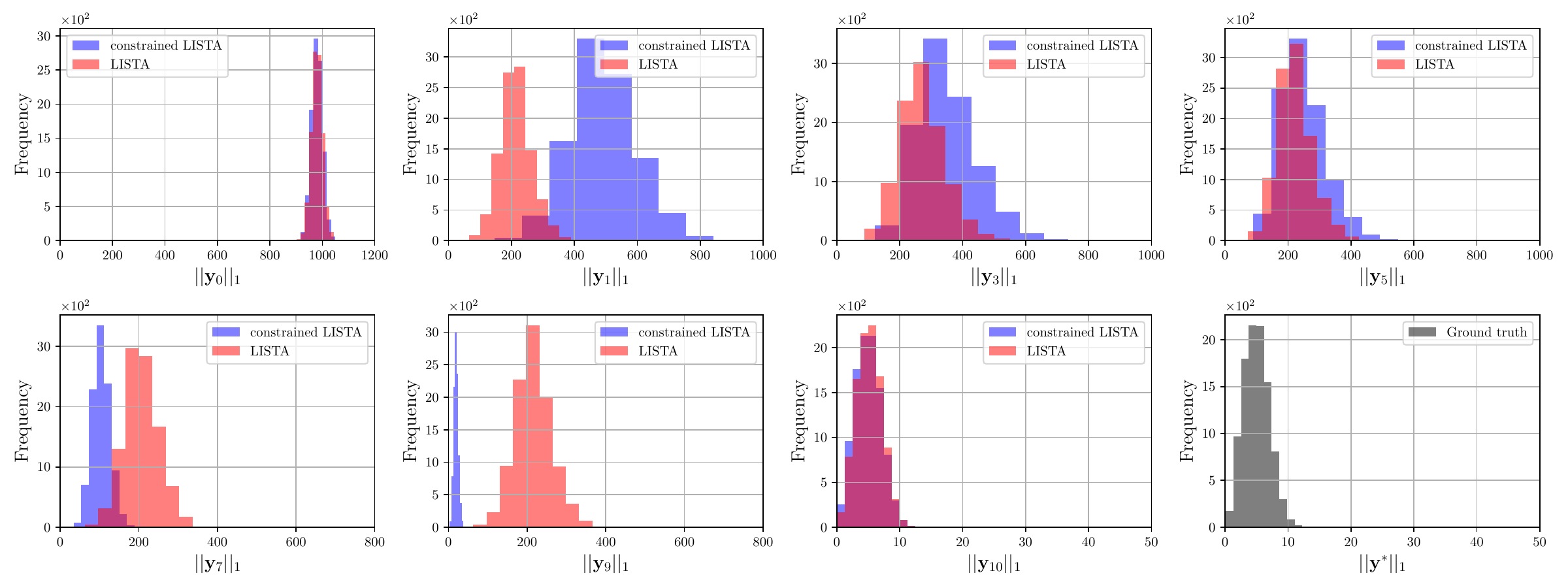}
    \caption{Histogram of the $\ell_1$-norm of the unrolled layers' outputs at (from top left to bottom right) the input, $1^{st}$, $3^{th}$, $5^{th}$, $7^{th}$, $9^{th}$ and $10^{th}$ layers along with the ground truth ${\bf y}^*$. The histogram stays the same across all the intermediate layers in the case of LISTA, while it moves to the left---representing lower values of $\ell_1$-norm---under the descending constraints.}
    \label{fig:hist}
\end{figure*}

\textbf{Performance.}
\cref{fig:distance_LISTA} compares the performance of LISTA and constrained LISTA on a test dataset. Both models achieve equivalent performance at their outputs, illustrated by the zero distance achieved at the last layer in both cases. However, their behaviors over the intermediate layers differ significantly from each other. For standard LISTA, the distance to the optimal is almost the same across all these layers, which suggests that there was no significant improvement in their estimates. This is further confirmed by examining the value of the objective function $f_{sp}({\bf y}, {\bf x})$ in \cref{fig:distance_LISTA} (right), which shows negligible change across the intermediate layers. A similar pattern is observed for other training procedures, such as incremental training \cite{Ito19} and layer-per-layer training \cite{Sabulal20}. In contrast, constrained LISTA demonstrates a gradual decrease in the distance to the optimum and the objective function because of the imposed constraints. 

\begin{figure*}
    \centering
    \includegraphics[width=0.7\textwidth]{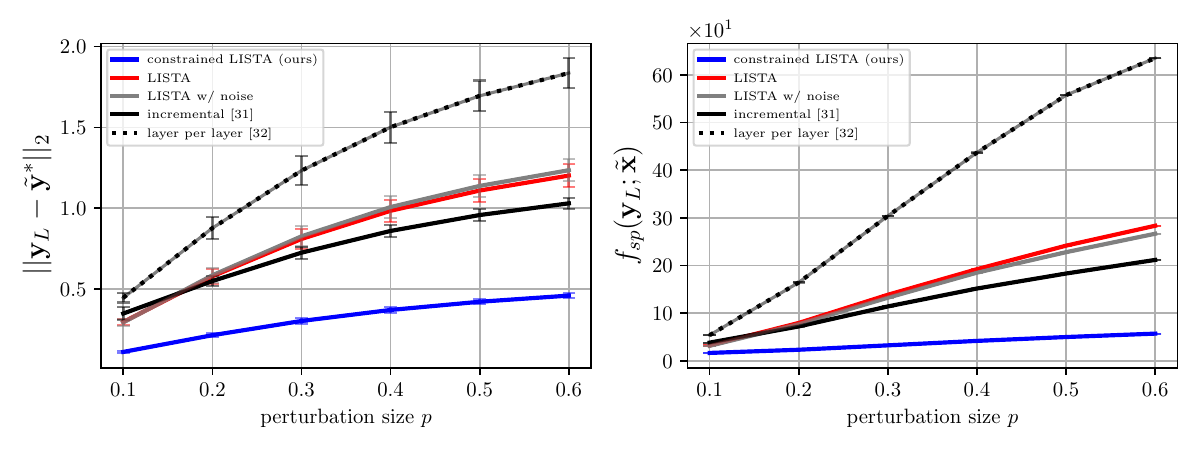}
    \caption{OOD robustness against data shifts of the form $\tilde{\bf x}  \sim {\cal N}({\bf x}, p^2{\bf I})$. Distance to the optimal solution $\tilde{\bf y}^*$ and the value of the objective function $f_{sp}({\bf y}_l)$ across the ten unrolled layers of constrained LISTA (blue), LISTA (red), LISTA trained with noisy inputs (gray), and other benchmarks (black). }
    \label{fig:OOD_LISTA}
\end{figure*}

To elaborate more on this difference, we show the histogram of the $\ell_1$-norm of the estimates in \cref{fig:hist}. As shown in the figure, the $\ell_1$-norm of the estimates barely changes between the first and ninth LISTA's layers (shown in red). This suggests that LISTA overlooks the sparsity requirement of the solution altogether, which is aligned with the fact that this requirement is not represented in the MSE training loss that is used in LISTA. Meanwhile, constrained LISTA uses the constraints \eqref{eq:distConst} to provoke this requirement through all the layers. Therefore, we witness in \cref{fig:hist} that the histogram of constrained LISTA is moving to the left over all the layers, implying a reduction in the $\ell_1$-norm at each and every layer.

These observations suggest that standard LISTA can find the solution in just two layers, where a significant reduction in the distance to the optimum is observed. In contrast, constrained LISTA requires more layers for a gradual reduction to occur. In order to examine whether we can combine the advantages of the two models--namely, achieving substantial progress in the first layer and then allowing the network to descend--we train LISTA with constraints starting from the second layer, leaving the first layer unconstrained. \cref{fig:distance_LISTA} shows that the model prioritizes learning a descending path over the constrained layers. However, unlike standard unrolling, the unconstrained first layer follows an ascending direction, where both the distance to the optimum and the the value of $f_{sp}$ increase. This implies that the direction predicted by an unconstrained layer is largely arbitrary.

\textbf{Robustness.} To evaluate the impact of the descending constraints, we test the unrolled network on an OOD distribution. The OOD dataset $\{(\tilde{\bf x}_i, \tilde{\bf y}_i^*)\}_i$ is constructed by adding perturbations to the input data ${\bf x} \sim D_x$ and then running ISTA from scratch to find the optimal sparse representations of the new signals $\tilde{\bf x}$. The perturbations are randomly drawn from a Gaussian distribution ${\cal N}(0, p^2 {\bf I})$ and the larger the level of perturbations $p$, the larger the distance $d(D_x, D_x')$. \cref{fig:OOD_LISTA} shows the robustness of constrained LISTA against the distribution shift, as indicated by the slow increase in the distance to the optimal as the perturbation size increases. In contrast, standard LISTA and the other benchmarks experience a higher distance to the optimum and a higher value of the objective function for all values of $p$. Additionally, the gap between the two models widen significantly with the increase in perturbation size. 

\begin{figure*}
    \centering
    \includegraphics[width=\textwidth]{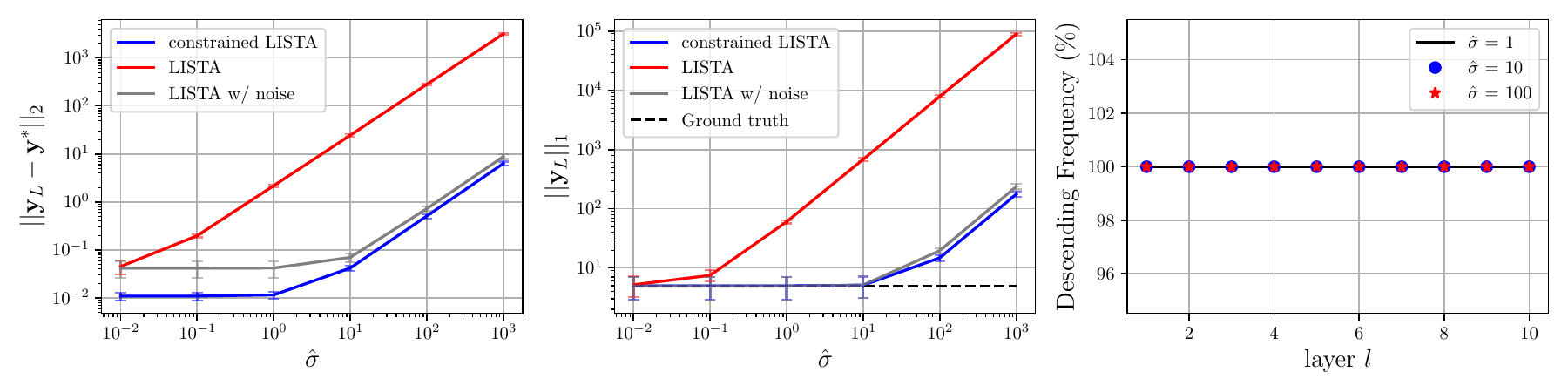}
    \caption{Robustness against additive noise with a standard deviation proportional to the gradient by a factor $\hat \sigma$ (training is held at $\hat \sigma=1$). (Left) Distance to the optimal ${\bf y}^*$, (middle) $\ell_1$-norm of the unrolled network's output, and (right) Frequency of satisfying the constraints across the test dataset for different perturbation sizes.}
    \label{fig:robust_LISTA}
\end{figure*}

To ensure a fair comparison, we also train standard LISTA with the same additive noise schedule introduced in Section \ref{sec:descend_constraint}. Although adding perturbations during training might suggest enhanced robustness of the model, \cref{fig:OOD_LISTA} shows that incorporating noise during training does not significantly improve the generalizability of standard LISTA. 
Furthermore, \cref{fig:robust_LISTA} illustrates the performance of both LISTA and constrained LISTA when subjected to additive perturbations at each unrolled layer during execution, similar to the noise introduced during training but with different magnitudes. 
The figure confirms that training the models with noise does confer some robustness toward this type of perturbations. However, noisy training alone in insufficient to ensure robustness toward distribution shifts in the input data.
This finding confirms that the OOD robustness observed in constrained LISTA, as shown in \cref{fig:OOD_LISTA}, is entirely attributed to the proposed descending constraints and the convergence guarantees they provide, even under distribution shifts.


\section{Case Study II: \\ Generative Flow Models for Inpainting Problems}\label{sec:inpainting}
For our second case study, we present the inpainting problem, which is one of the challenging inverse problems. Following the same structure of the previous section, we first define the inpainting problem, followed by a description of one method to unroll its solution with a generative model. We then discuss our constrained version of this solution and its robustness with the help of numerical simulations.

\subsection{Problem Formulation}
The inpainting problem is to predict image $\bf y$ from a measurement ${\bf x} = {\bf Ay}$, where $\bf A$ is a matrix masking the center elements of image $\bf y$. One common approach to solve this problem is through maximum a posteriori (MAP) inference, i.e., solving the optimization problem $\arg \max_{\bf y} p({\bf x}|{\bf y}) p_{\mu}({\bf y})$. The likelihood $p({\bf x}|{\bf y})$ in our case is a degenerate distribution that we approximate using a narrow Gaussian distribution ${\cal N}({\bf Ay}, \sigma_n^2{\bf I})$. Under this distribution, the problem reduces to the (negative log posterior) problem
\begin{equation}\label{eq:MAPinpainting}
     \begin{split}
        \arg \min_{{\bf y}} \quad & \frac{1}{2 \sigma_n^2} \|{\bf x} - {\bf Ay}\|_2^2 - \log p_{\mu}({\bf y}).
    \end{split}
\end{equation}
It is often the case that the prior distribution $p_{\mu}({\bf y})$ is carefully chosen based on domain knowledge or modeled using a learnable generative model.

In our case study, we model the prior using the generative flow (GLOW) models \cite{kingma2018glow}. GLOW transforms between a standard normal distribution $p({\bf z}) \sim {\cal N}(0, {\bf I})$ of a latent variable $\bf z$ and a more complex distribution using a series of \emph{composable}, \emph{bijective}, and \emph{differentiable} mappings (or layers). These functions map each image to its latent representation, i.e., 
\begin{equation}
    \begin{split}
        {\bf z} = h_{\mu}({\bf y}) = (h_1 \circ \dots \circ h_i)({\bf y}).
    \end{split}
\end{equation}
where $\circ$ denotes a composition of functions and $\mu$ are learnable parameters.
The inverse mappings, on the other hand, generate an image from a point in the latent space, i.e.,
\begin{equation}
    \begin{split}
        {\bf y} = g_{\mu}({\bf z}) = (h_i^{-1} \circ \dots \circ h_1^{-1})({\bf z}).
    \end{split}
\end{equation}
Each mapping/layer $h_*$ consists of actnorm, followed by an invertible $1 \times 1$ convolution, followed by a coupling layer. We refer the reader to Table 1 in \cite{kingma2018glow} for a precise description of these operations. 

Having these mappings learned leads to an exact evaluation of the prior distribution $p_\mu({\bf y})$ in terms of the latent distribution \cite{asim2020invertible, whang2021solving, Wei22}. Under the assumption that the latter is a standard normal distribution, the MAP problem in \eqref{eq:MAPinpainting} can be transferred to the latent space leading to the problem
\begin{equation}\label{eq:inpainting}
     \begin{split}
        \arg \min_{{\bf z}} \quad &  f_{in}({\bf z}; {\bf x}) := \|{\bf x} - {\bf A} g_{\mu}({\bf z})\|_2^2 + \lambda \|{\bf z}\|_2^2,
    \end{split}
\end{equation}
where $\lambda$ is a regularization parameter balancing the data consistency and the prior. Finding the optimal of \eqref{eq:inpainting} requires the generative model $g_{\mu}$ to be known apriori. This implies that learning the parameters of $g_{\mu}$ is executed separately before \eqref{eq:inpainting} is solved. We denote the data consistency term $\|{\bf x} - {\bf A} {\bf y}\|_2^2$ by $f_1({\bf y})$ and the regularization term $\|{\bf z}\|_2^2$ by $f_2({\bf z})$.

In this case study, we adopt the work in \cite{Wei22}, which takes an alternative path in solving \eqref{eq:inpainting} through unrolling an iterative proximal-like algorithm.

\subsection{Unrolled GLOW-Prox}
A solution for \eqref{eq:inpainting} can be reached using a proximal-like algorithm that alternates between taking a gradient step in the direction of data consistency and projecting the solution to the proximity of the prior. 
The work in \cite{Wei22}, which we refer to as unrolled GLOW-Prox, sets the parameters of this algorithm free to learn. More specifically, the $l$-th unrolled layer first updates the estimate using a gradient step,
\begin{equation}\label{eq:UGLOW1}
    \tilde{\bf y}_{l} = {\bf y}_{l-1} - \xi_{l} {\bf A}^\top ({\bf x} - {\bf Ay}_{l-1}),
\end{equation}
where $\xi_{l}$ is a learnable step size and the initial estimate ${\bf y}_0$ is set to be the masked image ${\bf x}$. This estimate is then converted to the latent space using a learnable GLOW network,
\begin{equation}
    \tilde{\bf z}_l = h_{\mu}^{(l)}(\tilde{\bf y}_l).
\end{equation}
This conversion is carried out so that a proximal update of the latent variable is executed next,
\begin{equation}\label{eq:UGLOW3}
    {\bf z}_l = \frac{\tilde{\bf z}_l}{1 + \zeta_{l}},
\end{equation}
where $\zeta_{l}$ is a learnable shrinkage parameter. The proximal update ensures that the estimate has a higher likelihood--equivalent to a lower $\ell_2$ norm since $\bf z$ follows a standard normal distribution--by bringing ${\bf z}_l$ closer to the origin. Finally, we convert back to the signal space using the inverse map $g_{\mu}^{(l)}$,
\begin{equation}\label{eq:UGLOW4}
    {\bf y}_l = g_{\mu}^{(l)}({\bf z}_l).
\end{equation}
Hence, the estimate ${\bf y}_l$ is a modified version of $\tilde{\bf y}_l$ that enhances the likelihood of the latent space.
The learnable parameters at each layer are the scalars $\xi_l$ and $\zeta_l$ and a full GLOW network $h_\mu^{(l)}$. Equations \eqref{eq:UGLOW1}-\eqref{eq:UGLOW4} are then repeated for $L$ layers and the whole unrolled network is trained end-to-end to generate a final estimate ${\bf y}_L = \boldsymbol{\Phi}({\bf x};{\bf W})$. 

\subsection{Constrained GLOW-Prox}
Unlike \cite{Wei22}, we perform a constrained training procedure according to \eqref{eq:constrainedUO}. The unrolled network is trained by minimizing the mean square error between the output of unrolled GLOW-Prox and the true image, i.e.,
\begin{equation}\label{eq:GLOW_training}
    \begin{split}
        \min_{{\bf W}} \quad &  \mathbb{E} \big[\|\boldsymbol{\Phi}({\bf x};{\bf W}) - {\bf y}\|_2^2 \big] \\
    \text{s.t.} \quad & \mathbb{E} \big[ \| {\nabla}_y f_1({\bf y}_{l})\|_2  -  (1-\epsilon) \ \|  {\nabla}_y f_1({\bf y}_{l-1}) \|_2 \big] \leq 0 \\
    & \mathbb{E} \big[ \| {\nabla}_z f_2({\bf z}_l)\|_2  -  (1-\epsilon) \ \|  {\nabla}_z f_2({\bf z}_{l-1}) \|_2 \big] \leq 0,
    \end{split}
\end{equation}
where $f_1$ and $f_2$ are the data consistency and the regularization terms specified in \eqref{eq:inpainting}.
Here, we use two sets of the gradient constraints \eqref{eq:gradConst} as we separate the variables ${\bf y}$ and ${\bf z}$ when computing the gradients, following the separation of their update rules in \eqref{eq:UGLOW1} and \eqref{eq:UGLOW3}. In the first set of constraints, we force the gradient of the data consistency term $f_1({\bf y})$ with respect to ${\bf y}$ to descend, thereby ensuring that the update rule of ${\bf z}$ in \eqref{eq:UGLOW3} does not drift the estimate ${\bf y}_l$ away from $\tilde{\bf y}_l$. 
In other words, $\tilde{\bf y}_l$ is, by definition, chosen in the direction of the gradient while ${\bf y}_l$ is forced to be in a descending direction using the imposed constraints. 
In the second set, we consider the gradients of the regularization term $f_2({\bf z})$ with respect to the latent variable ${\bf z}$ to ensure that its likelihood is increasing over the layers. The training procedure follows Algorithm \ref{alg:PD} except that two sets of dual variables are considered to accompany the two sets of constraints.
\begin{figure*}[t]
    \centering
    \includegraphics[width=\textwidth]
    {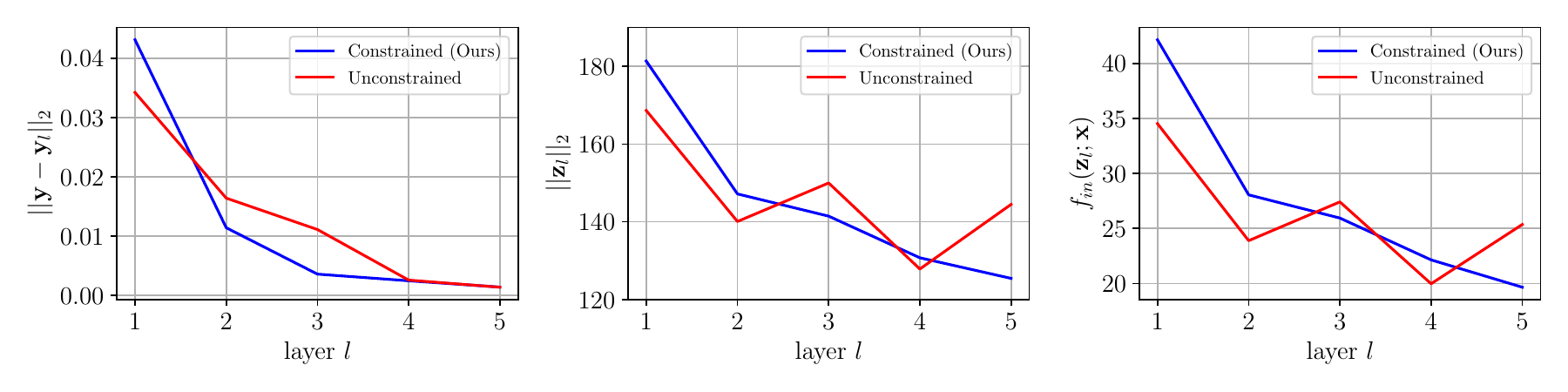}
    \caption{Comparison between constrained and unconstrained GLOW-Prox. (Left) Distance to the original image $\bf y$, averaged over the test dataset. (Middle) The $\ell_2$-norm of the latent variable ${\bf z}_l$, representing the negative log-likelihood. (Right) The value of the inpainting objective function $f_{in}(\cdot; {\bf x})$.}
    \label{fig:performace_GLOW}
\end{figure*}

\begin{figure*}[t]
    \centering
    \includegraphics[width=0.495\textwidth]
    {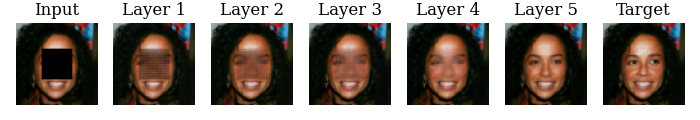}
    \includegraphics[width=0.495\textwidth]
    {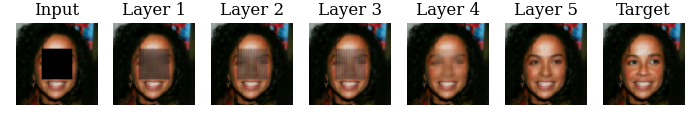}
    \includegraphics[width=0.495\textwidth]
    {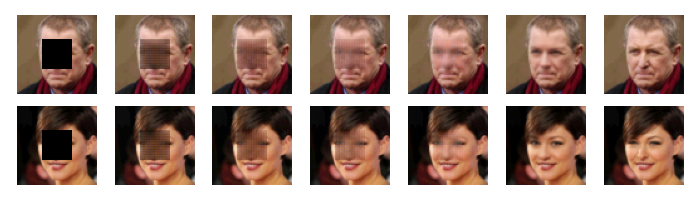}
    \includegraphics[width=0.495\textwidth]
    {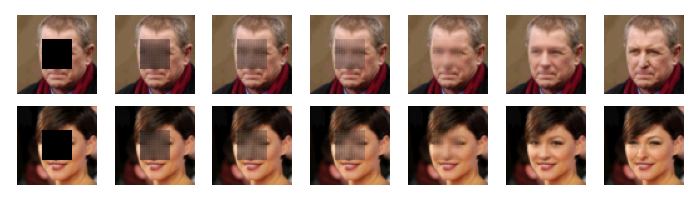}
    \caption{Test examples of inpainted images ${\bf y}_l$ across the unrolled layers of (left) constrained and (right) unconstrained GLOW-Prox. The performance of the two appears to be almost identical. The initial estimate ${\bf y}_0$ is set to be the masked image $\bf x$.}
    \label{fig:examples_GLOW}
\end{figure*}

\subsection{Numerical Results and Discussions}
We assess the proposed constrained-learning approach on the CelebA-HQ dataset \cite{karras2017progressive}. Each entry in the dataset is an RGB image with $64 \times 64$ pixels per color channel.
To construct the corrupted images, we mask the center $24 \times 24$ pixels, and the goal is to inpaint the masked pixels to match the clean images. To achieve this, the unrolled GLOW-Prox is trained using pairs of corrupted and clean images $\{({\bf x}_i, {\bf y}_i)\}_i$. The dataset consists of $17536$ training, $976$  validation and $976$  test examples.

The unrolled GLOW-Prox consists of $L=5$ layers. Each layer contains a GLOW network that has a depth of flow of $18$ and $4$ multi-scale levels \cite{kingma2018glow}. We refer the reader to \cite[Fig. 1]{Wei22} for a visualization of the network. 
We train the unrolled model to minimize \eqref{eq:GLOW_training} using ADAM with a learning rate $\mu_w = 10^{-5}$ and a dual learning rate $\mu_\lambda = 10^{-3}$. The constraint parameter $\epsilon$ is set to $0.05$. As per \eqref{eq:constrainedUO}, we add a noise vector ${\bf n}_l \sim {\cal N}(0, \sigma^2_l{\bf I})$ to the output of each unrolled layer $l$ and the initial estimate ${\bf y}_0$ is set to be the masked image $\bf x$. 
The noise variance declines over the layers, that is, $\sigma^2_l \propto \frac{1}{l}$.
The training is executed with mini-batches of size $8$ and was held for $30$ epochs.\footnote{The code is available at: \href{https://github.com/SMRhadou/UnrolledGlow}{https://github.com/SMRhadou/UnrolledGlow}.} We compare the performance and robustness of our constrained GLOW-Prox to the standard version trained without constraints \cite{Wei22}.

\textbf{Performance.} We evaluate the distance, averaged over the test dataset, between the output of each unrolled layer and the clean images. The results are reported in Figure \ref{fig:performace_GLOW} (left), which shows that both constrained and unconstrained GLOW-Prox have a similar behavior in the image domain. This is also confirmed in Figure \ref{fig:examples_GLOW}, which displays three test examples. As shown in the figure, the networks gradually inpaint the masked area reflecting a descending nature in the layers' estimates. The two methods, however, depart from each other in the latent space, as shown in Figure \ref{fig:performace_GLOW} (middle). The $\ell_2$-norm of the latent variable decreases gradually when the network is trained using constraints, implying that the gradient constraints were satisfied (and generalized) in the test examples. This is not the case when we omit the constraints during training, as depicted by the red line that fluctuates over the layers. Combined together, \cref{fig:performace_GLOW} (right) illustrates that the objective function $f_{in}$ indeed decreases over the layers in the case of constrained GLOW-prox. 

\begin{figure*}[t]
    \centering
    \begin{subfigure}[h]{0.39\textwidth}
        \includegraphics[width=\textwidth]{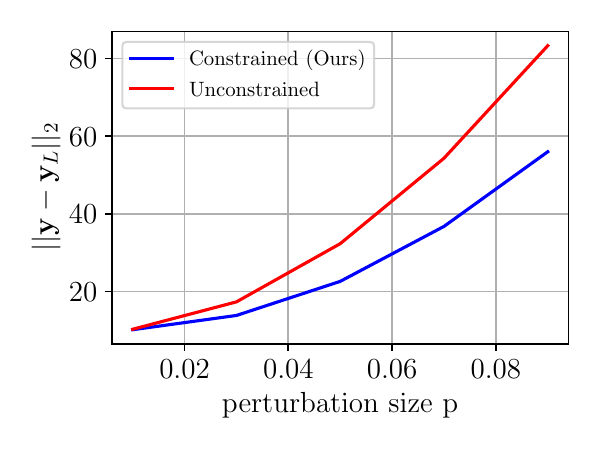}
    \end{subfigure}
    \begin{subfigure}[h]{0.58\textwidth}
        \includegraphics[width=0.97\textwidth]{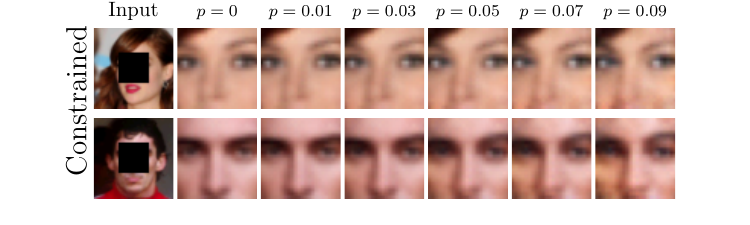}
        \includegraphics[width=0.97\textwidth]{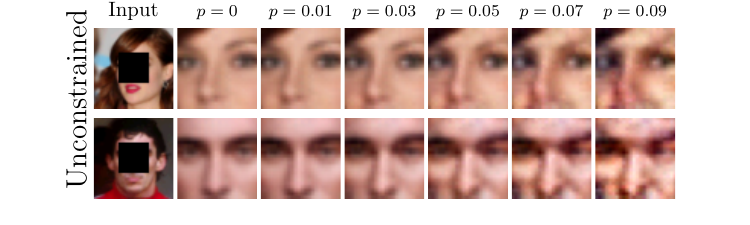}
    \end{subfigure}
    \caption{OOD Robustness. (Left) Mean square error between the original and reconstructed (inpainted) images under different perturbation sizes $p$. (Right) Examples of inpainted images using (top) constrained GLOW-Prox and (bottom) unconstrained GLOW-Prox for different $p$. Deterioration in the quality of the reconstructed images occurs faster in the unconstrained case.}
    \label{fig:robustness_GLOW}
\end{figure*}

\textbf{Robustness.} Constrained GLOW-Prox also shows more OOD robustness to input perturbations, as illustrated in \cref{fig:robustness_GLOW}. The left plot in the figure shows the MSE between the original and reconstructed images after adding perturbations to the input image. The perturbation is a random matrix ${\bf N}$, where each element in the matrix is sampled iid  sampled from a normal distribution ${\cal N}(0, p^2)$ and $p$ represents the perturbation size. As can be inferred from \cref{fig:robustness_GLOW} (left), the performance of the two models deteriorates with the perturbation size; however, GLOW-Prox deteriorates faster and more significantly than the constrained model. The distinction between the two models is more perceptible in \cref{fig:robustness_GLOW} (right), where the distortion in the shape of the facial features is more pronounced in the case of standard GLOW-Prox. In the case of constrained GLOW-Prox, the reconstructed images, albeit slightly different from the original, are still reasonably acceptable even under severe perturbations sizes.

\section{Conclusions}
 
In this paper, we introduced a framework for posing unrolled networks as stochastic descent algorithms. Within this framework, the unrolled layers are trained with descending constraints to ensure that their trajectory of estimates descends toward a stationary point of the optimizee problem. Our theoretical analysis shows that this trajectory is guaranteed to converge with an exponential rate of convergence. 
These substantiated convergence guarantees confer an advantage upon unrolled networks as we analytically show that the learnt descending behavior is transferable under distribution shifts in the optimization problem being solved. 
These findings collectively underscore the potential of the proposed framework to contribute to the development of robust, interpretable unrolled networks for a wider array of applications. One potential area for future work is the expansion of our framework to enforce the feasibility of the solutions in the context of constrained optimizee problems, which encompass many engineering problems.



\begin{appendices}
\section{Proof of Theorem \ref{thm:convergence}} \label{app:thm1}

We describe the notation we use in our analysis followed by the proof of Theorem \ref{thm:convergence}.
Consider a probability space $(\Omega, {\cal F}, P)$, where $\Omega$ is a sample space, ${\cal F}$ is a sigma algebra, and $P:{\cal F} \rightarrow [0,1]$ is a probability measure. We define a random variable $X: \Omega \rightarrow \mathbb{R}$ and write $P(\{ \omega: X(\omega) = 0\})$ as $P(X=0)$ to keep equations concise. We also define a filtration of $\cal F$ as $\{{\cal F}_l\}_{l>0}$, which can be thought of as an increasing sequence of $\sigma$-algebras with ${\cal F}_{l-1} \subset {\cal F}_l$. We assume that the outputs of the unrolled layers ${\bf y}_l$ are adapted to ${\cal F}_l$, i.e., ${\bf y}_l \in {\cal F}_l$. 

Moreover, a stochastic process $X_k$ is said to form a supermartingale if $\mathbb{E}[X_k | X_{k-1}, \dots, X_0] \leq X_{k-1}$. This inequality implies that given the past history of the process, the future value $X_k$ is not, on average, larger than the latest one. In the following, we provide a proof to Theorem \ref{thm:convergence}, which uses a supermartingale argument that is commonly used to prove convergence of stochastic descent algorithms.

\begin{proof}
Let $A_l \in {\cal F}_l$ be the event that the constraint \eqref{eq:constraints} at layer $l$ is satisfied. By the total expectation theorem, we have 
\begin{equation}\label{eq:totalexp}
\begin{split}
    \mathbb{E} & \big[ {\|{\nabla}  f({\bf y}_{l})\|}  \big] \\
    & = 
    P(A_l)\mathbb{E} \big[ {\|{\nabla}  f({\bf y}_{l})\|}  \ | A_l \big] 
    + P(A_l^c) \mathbb{E} \big[ {\|{\nabla}  f({\bf y}_{l})\|} \ | A_l^c \big]
\end{split}
\end{equation}
with $P(A_l)=1-\delta$. Note that we write $f({\bf y}_{l}; {\bf x})$ as $f({\bf y}_{l})$ for conciseness.
The first term on the right-hand side is the conditional expectation conditioned on the constraint being met and, in turn, is bounded above according to \eqref{eq:constraints}. The second term represents the complementary event $A_l^c \in {\cal F}_l$. The conditional expectation in this case can also be bounded since the (conditional) expectation of a random variable cannot exceed its maximum value, i.e., $\mathbb{E}\| {\nabla}  f({\bf y}_{l})\| \leq \max_{{\bf y}_l}\| {\nabla}  f({\bf y}_{l})\|$, and the latter is bounded by $M$ according to Assumption \ref{A1}. Consequently,  \eqref{eq:totalexp} is upper bounded by
\begin{equation}\label{eq:bound1}
\begin{split}
    & \mathbb{E} \big[ {\|{\nabla}  f({\bf y}_{l})\|}  \big] \\
    & \leq (1-\delta)(1-{\epsilon}) \ \mathbb{E} {\|  {\nabla}  f({\bf y}_{l-1})\|} + (1-\delta)\zeta(N, \delta) + \delta M,
\end{split}
\end{equation}
almost surely.

We let ${Z_l} = {\mathbb{E}_{D_x} \|  {\nabla}  f({\bf y}_{l})\|}$ be the gradient norm averaged over the input data distribution $D_x$, which is a random variable depending on the value of ${\bf y}_{l-1}$ and ${\bf n}_l$. We then construct a supermartingale inequality from \eqref{eq:bound1}:
\begin{equation}\label{eq:bound2}
\begin{split}
    & \mathbb{E}_{{\cal N}_l} [Z_l | \ {\cal F}_{l-1} ] 
     \leq (1-\delta)(1-{\epsilon}) \  Z_{l-1} + (1-\delta)\zeta(N, \delta) + \delta M\\
    & = (1-\delta) \  Z_{l-1} - (1-\delta) \Big( \epsilon Z_{l-1} - \zeta(N, \delta) - \frac{\delta M}{1-\delta} \Big).
\end{split}
\end{equation}
The conditional expectation on the left-hand side is taken over the distribution of the noise ${\bf n}_l$.
The goal of the rest of the proof is to show that i) when $l$ grows, $Z_l$ almost surely and infinitely often achieves values below $\eta := \frac{1}{\epsilon} \big(\zeta(N, \delta) + \delta M/1-\delta\big)$, and that ii) this also holds for the gradient norm $\|  {\nabla}  f({\bf y}_{l})\|$ itself. 

To achieve the first objective, it suffices to show that the lowest gradient norm achieved, on average, up to layer $l$ converges to a region of size $\eta$, i.e.,
\begin{equation}\label{eq:goal1}
    \lim_{l \rightarrow \infty}  \min_{k\leq l} \{Z_k \} \leq \eta \quad a.s.
\end{equation}
To prove the above inequality, we define the sequences
\begin{equation}\label{eq:sequences}
    \begin{split}
        \alpha_l & :=  Z_l \cdot \mathbf{1}\{  Z_l^\text{best} > \eta \},\\
        \beta_l & := \Big( \epsilon Z_{l} - \zeta(N, \delta) - \frac{\delta M}{1-\delta} \Big) \mathbf{1}\{  Z_l^\text{best} > \eta \},
    \end{split}
\end{equation}
where $Z_l^\text{best} = \min_{k\leq l} \{Z_k \}$ and $\mathbf{1}\{.\}$ is an indicator function.
The first sequence $\alpha_l$ 
keeps the values of $Z_l$ up until the best value $Z_l^\text{best}$ drops below $\eta$ for the first time. After this point, the best value stays below the threshold $\eta$ since $Z_{l+1}^\text{best} \leq Z_l^\text{best}$ by definition. This ensures that the indication function stays at zero (i.e., $\alpha_l = 0$) for the rest of the sequence. 
 Similarly, the sequence $\beta_l$ follows the values of $\epsilon Z_{l} - \zeta(N, \delta) - \frac{\delta}{1-\delta} M$ until it falls below zero for the first time. 

We now aim to show that  $\alpha_l$ forms a supermartingale, which requires finding an upper bound of the conditional expectation $\mathbb{E}[\alpha_l | {\cal F}_{l-1}]$. We use the total expectation theorem to write
\begin{equation} \label{eq:SM}
\begin{split}
    \mathbb{E}[\alpha_l | {\cal F}_{l-1}] & = \mathbb{E}[\alpha_l | {\cal F}_{l-1}, \alpha_{l-1}=0] P(\alpha_{l-1}=0) \\
    & + \mathbb{E}[\alpha_l | {\cal F}_{l-1}, \alpha_{l-1} \neq 0] P(\alpha_{l-1}\neq 0),
\end{split}
\end{equation}
splitting the expectation into two cases: $\alpha_{l-1}=0$ and $\alpha_{l-1} \neq 0$.
When $\alpha_{l-1}=0$, \eqref{eq:sequences} implies that either $Z_{l}$ or the indicator function is zero (i.e., $Z_l^\text{best} \leq \eta$). However, $Z_l$ cannot be zero without $Z_l^\text{best} \leq \eta$, and, therefore, $\alpha_{l-1}=0$ always implies that the indicator function is zero and $Z_l^\text{best} \leq \eta$. It also follows that $\beta_{l-1}$ is zero when $\alpha_{l-1}=0$ since it employs the same indicator function. As we discussed earlier, once $\alpha_{l-1} = 0$, all the values that follow is also zero, i.e., $\alpha_{k}=0, \ \forall k\geq l-1$. Hence, the conditional expectation of $\alpha_l$ can be written as
\begin{equation}\label{eq:SM_part1}
    \mathbb{E}[\alpha_l | {\cal F}_{l-1}, \alpha_{l-1}=0] = (1-\delta)(\alpha_{l-1} - \beta_{l-1}) = 0.
\end{equation}
On the other hand, when $\alpha_{l-1}\neq 0$, the conditional expectation follows from the definition in \eqref{eq:sequences},
\begin{equation} \label{eq:SM_part2}
    \begin{split}
        \mathbb{E}[\alpha_l | & {\cal F}_{l-1}, \alpha_{l-1} \neq 0] \\
        & = \mathbb{E}[Z_l \cdot \mathbf{1}\{  Z_l^\text{best} > \eta \} | {\cal F}_{l-1}, \alpha_{l-1} \neq 0]\\
        & \leq \mathbb{E}[Z_l | {\cal F}_{l-1}, \alpha_{l-1} \neq 0]\\
        & \leq (1-\delta) \  Z_{l-1} - (1-\delta) \Big( \epsilon Z_{l-1} - \zeta(N, \delta) - \frac{\delta M}{1-\delta} \Big)\\
        & = (1-\delta) (\alpha_{l-1} - \beta_{l-1}).
    \end{split}
\end{equation}
The first inequality is true since the indicator function is at most one and the second inequality is a direct application of \eqref{eq:bound2}. The last equality results from that fact that the indicator function $\mathbf{1}\{ Z_l^\text{best} > \eta \}$ is one since $\alpha_{l-1} \neq 0$, which implies that $\alpha_{l-1} = Z_{l-1}$ and $\beta_{l-1} = \epsilon Z_{l-1} - \zeta(N, \delta) - \frac{\delta}{1-\delta} M$. 

Combining the results in \eqref{eq:SM_part1} and \eqref{eq:SM_part2} and substituting in \eqref{eq:SM}, it finally follows that 
\begin{equation}\label{eq:final_SM}
\begin{split}
     \mathbb{E} [\alpha_l | & {\cal F}_{l-1}] \\
    & \leq (1-\delta) (\alpha_{l-1} - \beta_{l-1}) [P(\alpha_{l-1}= 0) + P(\alpha_{l-1}\neq 0)]\\
    & = (1-\delta) (\alpha_{l-1} - \beta_{l-1}).
\end{split}
\end{equation}
It is worth noting that, for all $l$, $\alpha_l \geq 0$ and $\beta_l \geq 0$ by construction.
It then follows from supermartingale convergence theorem \cite[Theorem 1]{robbins_convergence_1971} that \eqref{eq:final_SM} implies that (i) $\alpha_l$ converges almost surely, and (ii) $\sum_{l=1}^\infty \beta_l$ is almost surely summable (i.e., finite). When the latter is written explicitly, we get
\begin{equation}\label{eq:finiteSum}
    \sum_{l=1}^\infty \Big( \epsilon Z_{l} - \zeta(N, \delta) - \frac{\delta M}{1-\delta} \Big) \mathbf{1}\{  Z_l^\text{best} > \eta \} < \infty, \quad a.s.,
\end{equation}
The almost sure convergence of the above sequence implies that the limit inferior and limit superior coincide and 
\begin{equation}\label{eq:liminfTotal}
    \liminf_{l \rightarrow \infty} \Big( \epsilon Z_{l} - \zeta(N, \delta) - \frac{\delta M}{1-\delta} \Big) \mathbf{1}\{  Z_l^\text{best} > \eta \} = 0, \quad a.s.
\end{equation}
Equation \eqref{eq:liminfTotal} is true if either there exist a sufficiently large $l$ such that $Z_l^\text{best} \leq \eta = \frac{1}{\epsilon} \big(\zeta(N, \delta) + \delta M/1-\delta\big)$ to set the indicator to zero or it holds that
\begin{equation}\label{eq:liminf}
     \liminf_{l \rightarrow \infty}  \Big( \epsilon Z_{l} - \zeta(N, \delta) - \frac{\delta M}{1-\delta} \Big) = 0, \quad a.s.
\end{equation}
which is equivalent to having $\sup_{l} \inf_{m\geq l}  Z_{m} = \frac{1}{\epsilon} \big(\zeta(N, \delta) + \frac{\delta M}{1-\delta} \big)$. Hence, there exists some large $l$ where $Z_l^\text{best} \leq \sup_{l} \inf_{m\geq l}  Z_{m}$, which leads to the same upper bound. This proves the correctness of \eqref{eq:goal1}.

To this end, we have shown the convergence of $Z_l^\text{best}$, which was defined as the best \textit{expected} value of $\| {\nabla}  f({\bf y}_{l})\|$. Now, we turn to show the convergence of the random variable $\| {\nabla}  f({\bf y}_{l})\|$ itself. We use that the fact that $Z_l = \int \| {\nabla}  f({\bf y}_{l})\| dP$ to re-write \eqref{eq:liminf} as
\begin{equation}
    \liminf_{l \rightarrow \infty}  \int \epsilon \| {\nabla}  f({\bf y}_{l})\|  dP = \zeta(N, \delta) + \frac{\delta M}{1-\delta}, \quad a.s.
\end{equation}
By Fatou's lemma \cite[Theorem 1.5.5]{durrett2019probability}, it follows that 
\begin{equation}
\begin{split}
    \int \liminf_{l \rightarrow \infty} \epsilon \| {\nabla}  f({\bf y}_{l})\| dP & \leq \liminf_{l \rightarrow \infty}  \int \epsilon \| {\nabla}  f({\bf y}_{l})\| dP \\
    & = \zeta(N, \delta) + \frac{\delta M}{1-\delta}.
\end{split}
\end{equation}
We can bound the left-hand side from below by defining $f^\text{best}_l := \min_{k\leq l} \| {\nabla}  f({\bf y}_{k})\|$ as the lowest gradient norm achieved up to layer $l$. By definition, $f^\text{best}_l \leq \liminf_{l \rightarrow \infty} \| {\nabla}  f({\bf y}_{l})\|$ for sufficiently large $l$. Therefore, we get
\begin{equation}
\begin{split}
    \epsilon \int  f^\text{best}_l dP & \leq \epsilon \int \liminf_{l \rightarrow \infty}  \| {\nabla}  f({\bf y}_{l})\|  dP \\
    & \leq  \zeta(N, \delta) + \frac{\delta M}{1-\delta}, \quad a.s.
\end{split}
\end{equation}
for some large $l$. Equivalently, we can write that
\begin{equation}\label{eq:res1}
        \lim_{l \rightarrow \infty} \int  f^\text{best}_l dP \leq \frac{1}{\epsilon} \left( \zeta(N, \delta) + \frac{\delta M}{1-\delta} \right), \quad a.s.
\end{equation}
which completes the proof.
\end{proof}

\section{Proof of Lemma \ref{thm:rate}} \label{app:rate}
Similar to the proof in Appendix \ref{app:thm1}, we write $f({\bf y}_{l}; {\bf x})$ as $f({\bf y}_{l})$ for conciseness.
\begin{proof}
First, we recursively unroll the right-hand side of \eqref{eq:bound1} to evaluate the reduction in the gradient norm $\mathbb{E} {\|{\nabla} f({\bf y}_{l})\|}$ after $l$ layers. The gradient norm at the $l$-th layer then satisfies 
\begin{equation}\label{eq:bound_recursive}
\begin{split}
    \mathbb{E} \big[ \|{\nabla} & f({\bf y}_{l})\| \big] 
    \leq \ \ (1-\delta)^l(1-{\epsilon})^l \ \mathbb{E} {\|  {\nabla} f({\bf y}_{0})\|}\\
    & + \sum_{i=0}^{l-1} (1-\delta)^{i-1}(1-{\epsilon})^{i-1} \Big[ (1-\delta)\zeta(N, \delta) + \delta M \Big].
\end{split}
\end{equation}
The summation on the right-hand side resembles a geometric sum, which can be simplified to
\begin{equation}\label{eq:bound_recursive2}
\begin{split}
    \mathbb{E} \big[ \|{\nabla} f({\bf y}_{l}) & \| \big] 
    \leq \ \ (1-\delta)^l(1-{\epsilon})^l \ \mathbb{E} {\|  {\nabla} f({\bf y}_{0})\|}\\
    & +\frac{1 - (1-\delta)^l(1-{\epsilon})^l}{1-(1-\delta)(1-{\epsilon})} \Big[ (1-\delta)\zeta(N, \delta) + \delta M \Big].
\end{split}
\end{equation}

Second, we aim to evaluate the distance between the gradient norm at the $L$-th layer and its optimal value,
\begin{equation}\label{eq:convergence}
\begin{split}
    \Big| \mathbb{E} \big[ {\|{\nabla} f({\bf y}_{L})\|}  \big] - & \mathbb{E} \big[ {\|{\nabla} f({\bf y}^*)\|} \big] \Big| \\
    =  \lim_{l \to \infty} \Big| & \mathbb{E} \big[ {\|{\nabla} f({\bf y}_{L})\|} \big] 
    - \mathbb{E} [ \min_{k\leq l} \|{\nabla} f({\bf y}_{k})\| ] \Big. \\
    & \Big.+ \mathbb{E} [ \min_{k\leq l} \|{\nabla} f({\bf y}_{k})\| ]
    -  \mathbb{E} \big[ {\|{\nabla} f({\bf y}^*)\|} \big] \Big|.
\end{split}
\end{equation}
In \eqref{eq:convergence}, we add and subtract $\lim_{l \to \infty} \mathbb{E} [ \min_{k\leq l} \|{\nabla} f({\bf y}_{k})\| ]$ in the right-hand side while imposing the limit when $l$ goes to infinity. Using triangle inequality, we re-write \eqref{eq:convergence} as
\begin{equation}\label{eq:convergence2}
\begin{split}
    \Big| \mathbb{E} \big[ \|{\nabla} f({\bf y}_{L})\| \big] & - \mathbb{E} \big[ {\|{\nabla} f({\bf y}^*)  \|} \big]\Big| \\
    & \leq  \lim_{l \to \infty} \Big| \mathbb{E} \big[ {\|{\nabla} f({\bf y}_{L})\|} \big] 
    -  \mathbb{E} [ \min_{k\leq l} \|{\nabla} f({\bf y}_{k})\| ] \Big| \\
    & \quad + \lim_{l \to \infty}  \Big|  \mathbb{E} [ \min_{k\leq l} \|{\nabla} f({\bf y}_{k})\| ]
    -  \mathbb{E} \big[ {\|{\nabla} f({\bf y}^*)\|} \big] \Big|.
\end{split}
\end{equation}
Note that the gradient of $f$ at the stationary point ${\bf y}^*$ is zero. Therefore, the second term on the right-hand side is upper bounded according to Theorem \ref{thm:convergence}. 

The final step required to prove Lemma \ref{thm:rate} is to evaluate the first term in \eqref{eq:convergence2}. To do so, we observe that 
\begin{equation}\label{eq:convergence3}
\begin{split}
    \lim_{l \to \infty} \Big| \mathbb{E} \big[ \|{\nabla} f({\bf y}_{L})\| \big] &
    -  \mathbb{E} [ \min_{k\leq l} \|{\nabla} f({\bf y}_{k})\| ] \Big| = \\
     & \mathbb{E} \big[ {\|{\nabla} f({\bf y}_{L})\|} \big] 
    - \lim_{l \to \infty} \mathbb{E} [ \min_{k\leq l} \|{\nabla} f({\bf y}_{k})\| ].
\end{split}
\end{equation}
This is the case since $ \mathbb{E} \big[ {\|{\nabla} f({\bf y}_{l})\|} \big] $ cannot go below the minimum of the gradient norm when $l$ goes to infinity. Next, we substitute \eqref{eq:bound_recursive2} in \eqref{eq:convergence3}
\begin{equation}\label{eq:convergence4}
    \begin{split}
         \Big| \mathbb{E} \big[ \|{\nabla} & f({\bf y}_{L})\| \big] 
    - \lim_{l \to \infty} \mathbb{E} [ \min_{k\leq l} \|{\nabla} f({\bf y}_{k})\| ] \Big| \\ & =
    (1-\delta)^L(1-{\epsilon})^L \ \mathbb{E} {\|  {\nabla} f({\bf y}_{0})\|} \\
    & +\frac{1 - (1-\delta)^L(1-{\epsilon})^L}{1-(1-\delta)(1-{\epsilon})} \Big[ (1-\delta)\zeta(N, \delta) + \delta M \Big] \\
    & - \lim_{l \to \infty} (1-\delta)^l(1-{\epsilon})^l \ \mathbb{E} {\|  {\nabla} f({\bf y}_{0})\|} \\
    & - \lim_{l \to \infty} \frac{1 - (1-\delta)^l(1-{\epsilon})^l}{1-(1-\delta)(1-{\epsilon})} \Big[ (1-\delta)\zeta(N, \delta) + \delta M \Big]
    .
    \end{split}
\end{equation}
The first limit in \eqref{eq:convergence4} goes to zero since $(1-\delta)(1-{\epsilon}) < 1$, and the second limit is evaluated as the constant $\frac{(1-\delta)\zeta(N, \delta) + \delta M}{1-(1-\delta)(1-{\epsilon})}$. Therefore, we get
\begin{equation}\label{eq:convergence5}
    \begin{split}
         \Big| \mathbb{E} \big[ \|{\nabla} f({\bf y}_{L})\| \big] &
    - \lim_{l \to \infty} \mathbb{E} [ \min_{k\leq l} \|{\nabla} f({\bf y}_{k})\| ] \Big| \\ 
     & = (1-\delta)^L(1-{\epsilon})^L \ \mathbb{E} {\|  {\nabla} f({\bf y}_{0})\|}\\
    & - \frac{(1-\delta)^L(1-{\epsilon})^L}{1-(1-\delta)(1-{\epsilon})} \Big[ (1-\delta)\zeta(N, \delta) + \delta M \Big] \\
    & \leq (1-\delta)^L(1-{\epsilon})^L \ \mathbb{E} {\|  {\nabla} f({\bf y}_{0})\|}.
    \end{split}
\end{equation}
The last inequality follows since the term $\frac{(1-\delta)^L(1-{\epsilon})^L}{1-(1-\delta)(1-{\epsilon})} [ (1-\delta)\zeta(N, \delta) + \delta M ]$ is nonnegative.
Combining the two results, i.e., \eqref{eq:convergence2} and \eqref{eq:convergence5}, we can get the upper bound
\begin{equation}\label{eq:convergence_final}
\begin{split}
    \Big| \mathbb{E} \big[ \|{\nabla} f({\bf y}_{L})\| & \big] - \mathbb{E} \big[ \|{\nabla} f({\bf y}^*)  \|  \big]\Big| \\
     & \leq (1-\delta)^L (1-{\epsilon})^L \ \mathbb{E} {\|  {\nabla} f({\bf y}_{0})\|} \\ 
     & \quad+ \frac{1}{\epsilon} \left( \zeta(N, \delta) + \frac{\delta M}{1-\delta} \right),
\end{split}
\end{equation}
which completes the proof.
\end{proof}

\section{Proof of Corollary \ref{cor:OOD}} \label{app:OOD}
\begin{proof}
We start by adding and subtracting $\mathbb{E}_{D_x} \big[ \|  {\nabla} f({\bf y}_{l};{\bf x}) \|_2 \big]$ and $(1-\epsilon) \Big[ \mathbb{E}_{D_x} \big[ \| {\nabla} f({\bf y}_{l-1};{\bf x})\|_2]$ from the quantity we seek to evaluate, i.e., we get
\begin{equation}
    \begin{split}
        &\ \mathbb{E}_{D_x'} \big[ \| {\nabla} f({\bf y}_{l};{\bf x})\|_2] - (1-\epsilon) \ \mathbb{E}_{D_x'} \big[ \|  {\nabla} f({\bf y}_{l-1};{\bf x}) \|_2 \big] \\
        & = \mathbb{E}_{D_x'} \big[ \| {\nabla} f({\bf y}_{l};{\bf x})\|_2] -  \ \mathbb{E}_{D_x} \big[ \|  {\nabla} f({\bf y}_{l};{\bf x}) \|_2 \big] \\
        & + (1-\epsilon) \Big[ \mathbb{E}_{D_x} \big[ \| {\nabla} f({\bf y}_{l-1};{\bf x})\|_2] -  \ \mathbb{E}_{D_x'} \big[ \|  {\nabla} f({\bf y}_{l-1};{\bf x}) \|_2 \big] \Big] \\
        & + \mathbb{E}_{D_x} \big[ \| {\nabla} f({\bf y}_{l};{\bf x})\|_2] - (1-\epsilon) \ \mathbb{E}_{D_x} \big[ \|  {\nabla} f({\bf y}_{l-1};{\bf x}) \|_2 \big].
    \end{split}
\end{equation}

The right-hand side consists of three terms that can be bounded above with positive quantities according to Assumption \ref{A6} and \eqref{eq:constraints}. Therefore, the 

\begin{equation}
    \begin{split}
        \mathbb{E}_{D_x'} \big[ & \| {\nabla} f({\bf y}_{l};{\bf x})\|_2] - (1-\epsilon) \ \mathbb{E}_{D_x'} \big[ \|  {\nabla} f({\bf y}_{l-1};{\bf x}) \|_2 \big] \\
        & \leq   M d(D_x, D_x') + M (1-\epsilon) d(D_x, D_x') + \zeta(N, \delta) \\
        & \leq  2M d(D_x, D_x') + \zeta(N, \delta).
    \end{split}
\end{equation}
Notice that this inequality holds with probability $1-\delta$ since the upper bound in \eqref{eq:constraints} also holds with  the same probability.
This completes the proof.
    
\end{proof}

\end{appendices}

\bibliographystyle{ieeetr}
\bibliography{Bib}

\begin{thebibliography}{10}

\bibitem{gregor_learning_2010}
K.~Gregor and Y.~LeCun, ``Learning fast approximations of sparse coding,'' in
  {\em Proceedings of the 27th {International} {Conference} on {International}
  {Conference} on {Machine} {Learning}}, {ICML}'10, pp.~399--406, June 2010.

\bibitem{yang2022transformers}
Y.~Yang, D.~P. Wipf, {\em et~al.}, ``Transformers from an optimization
  perspective,'' {\em Advances in Neural Information Processing Systems},
  vol.~35, pp.~36958--36971, 2022.

\bibitem{yu2023white}
Y.~Yu, S.~Buchanan, D.~Pai, T.~Chu, Z.~Wu, S.~Tong, B.~D. Haeffele, and Y.~Ma,
  ``White-box transformers via sparse rate reduction,'' {\em arXiv preprint
  arXiv:2306.01129}, 2023.

\bibitem{von2023transformers}
J.~Von~Oswald, E.~Niklasson, E.~Randazzo, J.~Sacramento, A.~Mordvintsev,
  A.~Zhmoginov, and M.~Vladymyrov, ``Transformers learn in-context by gradient
  descent,'' in {\em International Conference on Machine Learning},
  pp.~35151--35174, PMLR, 2023.

\bibitem{monga_algorithm_2021}
V.~Monga, Y.~Li, and Y.~C. Eldar, ``Algorithm unrolling: Interpretable,
  efficient deep learning for signal and image processing,'' {\em IEEE Signal
  Processing Magazine}, vol.~38, pp.~18--44, Mar. 2021.

\bibitem{zhang2020deep}
K.~Zhang, L.~V. Gool, and R.~Timofte, ``Deep unfolding network for image
  super-resolution,'' in {\em Proceedings of the IEEE/CVF conference on
  computer vision and pattern recognition}, pp.~3217--3226, 2020.

\bibitem{Wei22}
X.~Wei, H.~van Gorp, L.~Gonzalez-Carabarin, D.~Freedman, Y.~C. Eldar, and
  R.~J.~G. van Sloun, ``Deep unfolding with normalizing flow priors for inverse
  problems,'' {\em IEEE Transactions on Signal Processing}, vol.~70,
  pp.~2962--2971, 2022.

\bibitem{mou2022deep}
C.~Mou, Q.~Wang, and J.~Zhang, ``Deep generalized unfolding networks for image
  restoration,'' in {\em Proceedings of the IEEE/CVF Conference on Computer
  Vision and Pattern Recognition}, pp.~17399--17410, 2022.

\bibitem{Li20}
Y.~Li, M.~Tofighi, J.~Geng, V.~Monga, and Y.~C. Eldar, ``Efficient and
  interpretable deep blind image deblurring via algorithm unrolling,'' {\em
  IEEE Transactions on Computational Imaging}, vol.~6, pp.~666--681, 2020.

\bibitem{qiao2023towards}
P.~Qiao, S.~Liu, T.~Sun, K.~Yang, and Y.~Dou, ``Towards vision transformer
  unrolling fixed-point algorithm: a case study on image restoration,'' {\em
  arXiv preprint arXiv:2301.12332}, 2023.

\bibitem{hu2020iterative}
Q.~Hu, Y.~Cai, Q.~Shi, K.~Xu, G.~Yu, and Z.~Ding, ``Iterative algorithm induced
  deep-unfolding neural networks: Precoding design for multiuser {MIMO}
  systems,'' {\em IEEE Transactions on Wireless Communications}, vol.~20,
  no.~2, pp.~1394--1410, 2020.

\bibitem{chowdhury2021unfolding}
A.~Chowdhury, G.~Verma, C.~Rao, A.~Swami, and S.~Segarra, ``Unfolding {WMMSE}
  using graph neural networks for efficient power allocation,'' {\em IEEE
  Transactions on Wireless Communications}, vol.~20, no.~9, pp.~6004--6017,
  2021.

\bibitem{liu2021deep}
Y.~Liu, Q.~Hu, Y.~Cai, G.~Yu, and G.~Y. Li, ``Deep-unfolding beamforming for
  intelligent reflecting surface assisted full-duplex systems,'' {\em IEEE
  Transactions on Wireless Communications}, vol.~21, no.~7, pp.~4784--4800,
  2021.

\bibitem{Schynol23}
L.~Schynol and M.~Pesavento, ``Coordinated sum-rate maximization in multicell
  {MU-MIMO} with deep unrolling,'' {\em IEEE Journal on Selected Areas in
  Communications}, vol.~41, no.~4, pp.~1120--1134, 2023.

\bibitem{huang2023regularization}
H.~Huang, Y.~Lin, G.~Gui, H.~Gacanin, H.~Sari, and F.~Adachi, ``Regularization
  strategy aided robust unsupervised learning for wireless resource
  allocation,'' {\em IEEE Transactions on Vehicular Technology}, 2023.

\bibitem{yang2023knowledge}
H.~Yang, N.~Cheng, R.~Sun, W.~Quan, R.~Chai, K.~Aldubaikhy, A.~Alqasir, and
  X.~Shen, ``Knowledge-driven resource allocation for {D2D} networks: A {WMMSE}
  unrolled graph neural network approach,'' {\em arXiv preprint
  arXiv:2307.05882}, 2023.

\bibitem{li2021deep}
Y.~Li, O.~Bar-Shira, V.~Monga, and Y.~C. Eldar, ``Deep algorithm unrolling for
  biomedical imaging,'' {\em arXiv preprint arXiv:2108.06637}, 2021.

\bibitem{nakarmi2020multi}
U.~Nakarmi, J.~Y. Cheng, E.~P. Rios, M.~Mardani, J.~M. Pauly, L.~Ying, and
  S.~S. Vasanawala, ``Multi-scale unrolled deep learning framework for
  accelerated magnetic resonance imaging,'' in {\em 2020 IEEE 17th
  International Symposium on Biomedical Imaging (ISBI)}, pp.~1056--1059, IEEE,
  2020.

\bibitem{chennakeshava2022deep}
N.~Chennakeshava, T.~S. Stevens, F.~J. de~Bruijn, A.~Hancock, M.~Peka{\v{r}},
  Y.~C. Eldar, M.~Mischi, and R.~J. van Sloun, ``Deep proximal unfolding for
  image recovery from under-sampled channel data in intravascular ultrasound,''
  in {\em ICASSP 2022-2022 IEEE International Conference on Acoustics, Speech
  and Signal Processing (ICASSP)}, pp.~1221--1225, IEEE, 2022.

\bibitem{wang2023indudonet+}
H.~Wang, Y.~Li, H.~Zhang, D.~Meng, and Y.~Zheng, ``{InDuDoNet+}: A deep
  unfolding dual domain network for metal artifact reduction in ct images,''
  {\em Medical Image Analysis}, vol.~85, p.~102729, 2023.

\bibitem{hadou2023stochastic}
S.~Hadou, N.~NaderiAlizadeh, and A.~Ribeiro, ``Stochastic unrolled federated
  learning,'' {\em arXiv preprint arXiv:2305.15371}, 2023.

\bibitem{Ravi2016OptimizationAA}
S.~Ravi and H.~Larochelle, ``Optimization as a model for few-shot learning,''
  in {\em International Conference on Learning Representations}, 2016.

\bibitem{hershey2014deep}
J.~R. Hershey, J.~L. Roux, and F.~Weninger, ``Deep unfolding: Model-based
  inspiration of novel deep architectures,'' {\em arXiv preprint
  arXiv:1409.2574}, 2014.

\bibitem{nasser2022deep}
R.~Nasser, Y.~C. Eldar, and R.~Sharan, ``Deep unfolding for non-negative matrix
  factorization with application to mutational signature analysis,'' {\em
  Journal of Computational Biology}, vol.~29, no.~1, pp.~45--55, 2022.

\bibitem{noah2023limited}
Y.~Noah and N.~Shlezinger, ``Limited communications distributed optimization
  via deep unfolded distributed {ADMM},'' {\em arXiv preprint
  arXiv:2309.14353}, 2023.

\bibitem{Liu23}
C.~Liu, G.~Leus, and E.~Isufi, ``Unrolling of simplicial elasticnet for edge
  flow signal reconstruction,'' {\em IEEE Open Journal of Signal Processing},
  pp.~1--9, 2023.

\bibitem{Heaton_Chen_Wang_Yin_2023}
H.~Heaton, X.~Chen, Z.~Wang, and W.~Yin, ``Safeguarded learned convex
  optimization,'' {\em Proceedings of the AAAI Conference on Artificial
  Intelligence}, vol.~37, pp.~7848--7855, Jun. 2023.

\bibitem{shen2021learning}
J.~Shen, X.~Chen, H.~Heaton, T.~Chen, J.~Liu, W.~Yin, and Z.~Wang, ``Learning a
  minimax optimizer: A pilot study,'' in {\em International Conference on
  Learning Representations}, 2021.

\bibitem{Moeller_2019_ICCV}
M.~Moeller, T.~Mollenhoff, and D.~Cremers, ``Controlling neural networks via
  energy dissipation,'' in {\em Proceedings of the IEEE/CVF International
  Conference on Computer Vision (ICCV)}, October 2019.

\bibitem{liu2021investigating}
R.~Liu, P.~Mu, and J.~Zhang, ``Investigating customization strategies and
  convergence behaviors of task-specific {ADMM},'' {\em IEEE Transactions on
  Image Processing}, vol.~30, pp.~8278--8292, 2021.

\bibitem{Ito19}
D.~Ito, S.~Takabe, and T.~Wadayama, ``Trainable {ISTA} for sparse signal
  recovery,'' {\em IEEE Transactions on Signal Processing}, vol.~67, no.~12,
  pp.~3113--3125, 2019.

\bibitem{Sabulal20}
A.~P. Sabulal and S.~Bhashyam, ``Joint sparse recovery using deep unfolding
  with application to massive random access,'' in {\em ICASSP 2020 - 2020 IEEE
  International Conference on Acoustics, Speech and Signal Processing
  (ICASSP)}, pp.~5050--5054, 2020.

\bibitem{Chen18theoretical}
X.~Chen, J.~Liu, Z.~Wang, and W.~Yin, ``Theoretical linear convergence of
  unfolded ista and its practical weights and thresholds,'' in {\em Advances in
  Neural Information Processing Systems}, vol.~31, 2018.

\bibitem{liu2018alista}
J.~Liu, X.~Chen, Z.~Wang, and W.~Yin, ``{ALISTA}: Analytic weights are as good
  as learned weights in {LISTA},'' in {\em International Conference on Learning
  Representations}, 2019.

\bibitem{Xie19DL-ADMM}
X.~Xie, J.~Wu, G.~Liu, Z.~Zhong, and Z.~Lin, ``Differentiable linearized
  {ADMM},'' in {\em Proceedings of the 36th International Conference on Machine
  Learning}, vol.~97 of {\em Proceedings of Machine Learning Research},
  pp.~6902--6911, PMLR, 09--15 Jun 2019.

\bibitem{Abadi15}
M.~Abadi, A.~Chu, I.~Goodfellow, H.~B. McMahan, I.~Mironov, K.~Talwar, and
  L.~Zhang, ``Deep learning with differential privacy,'' in {\em Proceedings of
  the 2016 ACM SIGSAC Conference on Computer and Communications Security}, CCS
  '16, p.~308–318, 2016.

\bibitem{hadou2023space}
S.~Hadou, C.~I. Kanatsoulis, and A.~Ribeiro, ``Space-time graph neural networks
  with stochastic graph perturbations,'' in {\em IEEE International Conference
  on Acoustics, Speech and Signal Processing (ICASSP)}, pp.~1--5, 2023.

\bibitem{hadouspace}
S.~Hadou, C.~I. Kanatsoulis, and A.~Ribeiro, ``Space-time graph neural
  networks,'' in {\em International Conference on Learning Representations
  (ICLR)}, 2022.

\bibitem{gama20}
F.~Gama, J.~Bruna, and A.~Ribeiro, ``Stability properties of graph neural
  networks,'' {\em IEEE Transactions on Signal Processing}, vol.~68,
  pp.~5680--5695, 2020.

\bibitem{Andrychowicz16}
M.~Andrychowicz, M.~Denil, S.~G. Colmenarejo, M.~W. Hoffman, D.~Pfau,
  T.~Schaul, B.~Shillingford, and N.~de~Freitas, ``Learning to learn by
  gradient descent by gradient descent,'' in {\em Proceedings of the 30th
  International Conference on Neural Information Processing Systems},
  p.~3988–3996, 2016.

\bibitem{liu2022optimization}
R.~Liu, X.~Liu, S.~Zeng, J.~Zhang, and Y.~Zhang, ``Optimization-derived
  learning with essential convergence analysis of training and
  hyper-training,'' in {\em International Conference on Machine Learning},
  pp.~13825--13856, PMLR, 2022.

\bibitem{chamon2022constrained}
L.~F. Chamon, S.~Paternain, M.~Calvo-Fullana, and A.~Ribeiro, ``Constrained
  learning with non-convex losses,'' {\em IEEE Transactions on Information
  Theory}, 2022.

\bibitem{boyd2004convex}
S.~P. Boyd and L.~Vandenberghe, {\em Convex optimization}.
\newblock Cambridge university press, 2004.

\bibitem{Eksin12}
C.~Eksin and A.~Ribeiro, ``Distributed network optimization with heuristic
  rational agents,'' {\em IEEE Transactions on Signal Processing}, vol.~60,
  no.~10, pp.~5396--5411, 2012.

\bibitem{ISTA2004}
I.~Daubechies, M.~Defrise, and C.~De~Mol, ``An iterative thresholding algorithm
  for linear inverse problems with a sparsity constraint,'' {\em Communications
  on Pure and Applied Mathematics}, vol.~57, no.~11, pp.~1413--1457, 2004.

\bibitem{beck_fast_2009}
A.~Beck and M.~Teboulle, ``A fast iterative shrinkage-thresholding algorithm
  with application to wavelet-based image deblurring,'' in {\em 2009 {IEEE}
  {International} {Conference} on {Acoustics}, {Speech} and {Signal}
  {Processing}}, pp.~693--696, Apr. 2009.

\bibitem{chung16}
H.~Chung, S.~J. Lee, and J.~G. Park, ``Deep neural network using trainable
  activation functions,'' in {\em 2016 International Joint Conference on Neural
  Networks (IJCNN)}, pp.~348--352, 2016.

\bibitem{kiliccarslan2021rsigelu}
S.~Kili{\c{c}}arslan and M.~Celik, ``{RSigELU}: A nonlinear activation function
  for deep neural networks,'' {\em Expert Systems with Applications}, vol.~174,
  p.~114805, 2021.

\bibitem{varshney2021optimizing}
M.~Varshney and P.~Singh, ``Optimizing nonlinear activation function for
  convolutional neural networks,'' {\em Signal, Image and Video Processing},
  vol.~15, no.~6, pp.~1323--1330, 2021.

\bibitem{kingma2018glow}
D.~P. Kingma and P.~Dhariwal, ``{GLOW}: Generative flow with invertible 1x1
  convolutions,'' {\em Advances in neural information processing systems},
  vol.~31, 2018.

\bibitem{asim2020invertible}
M.~Asim, M.~Daniels, O.~Leong, A.~Ahmed, and P.~Hand, ``Invertible generative
  models for inverse problems: mitigating representation error and dataset
  bias,'' in {\em International Conference on Machine Learning}, pp.~399--409,
  PMLR, 2020.

\bibitem{whang2021solving}
J.~Whang, Q.~Lei, and A.~Dimakis, ``Solving inverse problems with a flow-based
  noise model,'' in {\em International Conference on Machine Learning},
  pp.~11146--11157, PMLR, 2021.

\bibitem{karras2017progressive}
T.~Karras, T.~Aila, S.~Laine, and J.~Lehtinen, ``Progressive growing of {GANs}
  for improved quality, stability, and variation,'' {\em arXiv preprint
  arXiv:1710.10196}, 2017.

\bibitem{robbins_convergence_1971}
H.~Robbins and D.~Siegmund, ``A convergence theorem for non negative almost
  supermartingales and some applications,'' in {\em Optimizing Methods in
  Statistics}, pp.~233--257, Academic Press, Jan. 1971.

\bibitem{durrett2019probability}
R.~Durrett, {\em Probability: theory and examples}, vol.~49.
\newblock Cambridge university press, 2019.

\end{thebibliography}

\end{document}